\documentclass[12pt]{article}

\usepackage[margin=2.5cm]{geometry}
\usepackage{times}
\usepackage{graphicx}
\usepackage{amsmath,amssymb,amsthm}
\usepackage{hyperref}
\usepackage[super]{natbib}
\usepackage{booktabs}
\usepackage{enumitem}
\usepackage{algorithm}
\usepackage{algpseudocode}
\usepackage{longtable}
\usepackage{CJKutf8}

\newtheorem{theorem}{Theorem}[section]
\newtheorem{lemma}[theorem]{Lemma}
\newtheorem{corollary}[theorem]{Corollary}

\newtheorem{definition}[theorem]{Definition}
\theoremstyle{remark}
\newtheorem{remark}[theorem]{Remark}
\newtheorem{example}[theorem]{Example}

\DeclareMathOperator{\Var}{Var}

\newcommand{\R}{\mathbb{R}}
\newcommand{\E}{\mathbb{E}}
\newcommand{\Pbb}{\mathbb{P}}

\newcommand{\LLM}{\mathrm{LLM}}
\newcommand{\Err}{\mathcal{E}}
\newcommand{\Irred}{\mathcal{E}_{\min}}

\title{Using Large Language Models as Low-Cost Statistical Estimators for Human-Response Data}

\author{Haobo Yang\\
Department of Computer Science and Engineering, SUSTech University\\
\href{mailto:yhbcode000@foxmail.com}{\texttt{yhbcode000@foxmail.com}}%
\thanks{This preprint has not undergone professional peer review. If you find any errors or have concerns, please contact the corresponding author.}}

\date{\today}

\begin{document}

\maketitle

\begin{abstract}
Quantitative research across the social and behavioral sciences depends on human subject experiments that are expensive, slow, and subject to sampling bias. Here we show that pretrained large language models induce risk-equivalent estimators of conditional expectations under squared loss, establishing restricted functional risk equivalence: under squared loss, the LLM induces an estimator whose risk matches the Bayes optimal risk for squared-loss prediction of conditional expectations for any inference that depends on the data only through the conditional mean. We formalize the LLM as a misspecified functional estimator $T(\hat{P}_n)$ trained on i.i.d.\ data, decompose the estimation error into representation bias $\epsilon_{\mathrm{rep}}$ and optimization error, and prove that under mild regularity conditions the LLM's expected error converges to the irreducible population variance plus the squared representation bias, with the representation bias bounded by the Pinsker inequality. The identifiability error $\delta$ propagates into the effective bias, inflating the asymptotic risk floor. We establish restricted functional risk equivalence via a bidirectional Le Cam deficiency analysis: the forward deficiency vanishes asymptotically while the reverse deficiency is exactly zero. We provide finite-sample concentration bounds and a calibration protocol with explicit decision rules. An open-source Python implementation is available at \url{https://github.com/yhbcode000/llmstat\_tool}. The result is a precise, provable statement: a well-calibrated LLM achieves the Bayes-optimal risk for conditional-mean-dependent inference, bounded by explicit scope conditions. In practical applications, this means that under satisfied conditions and well-calibrated models, large language models can be used in many prediction and decision-making tasks that originally relied on human experiments, approximating near-optimal statistical inference at lower cost.
\end{abstract}

\noindent\textbf{Area of research:} Artificial Intelligence, Statistical Learning Theory, Formal Methods, and Decision Theory.

\section{Introduction}

Human subject experiments are the cornerstone of quantitative research in the social, behavioural, and health sciences. A typical between-subjects study requires recruiting dozens to hundreds of participants per condition to detect meaningful effect sizes with adequate statistical power~\cite{cohen1988}. Recruitment is expensive, time-consuming, and bounded by participant availability, self-selection bias, and demographic skew~\cite{cohen1988}. Underpowered studies produce noisy estimates, and the cost of adequate sample sizes limits the scope and ambition of experimental research programs.

The emergence of large language models trained on vast corpora of
human-generated text raises a fundamental question: if an LLM has learned
statistical regularities of human behaviour from its training data, can its
conditional-mean estimator achieve the same squared-loss risk as estimators
built from human-subject data? This question is not merely speculative.
Researchers across disciplines have begun using LLMs to simulate survey
respondents, annotate behavioural data, and generate synthetic experimental
outcomes across psychology, political science, marketing, and health
research.

Existing approaches to validating LLM-based risk agreement rely exclusively
on empirical benchmarking: comparing LLM outputs to human baselines on
specific prompts and tasks. While these comparisons are informative, they
provide no formal guarantee that the risk agreement is mathematically
justified. A favourable comparison on one set of prompts does not guarantee
validity on another; a high correlation does not guarantee unbiased effect
size estimation; a small observed bias does not guarantee that the LLM
estimator converges to the relevant population functional.

We address this gap by constructing an axiomatic framework for LLM--human
functional risk equivalence under squared loss, and by constructing rigorous
proofs from first principles. Our approach has three distinguishing features.

First, all theorems are rigorous, relying only on stated axioms. We assume nothing about the LLM's architecture, training algorithm, or parameter count. The core requirements are that training data are representative enough for the learned conditional distribution to converge to the KL projection in the model class, that the conditional-mean functional is Lipschitz under the stated bounded-response assumptions, and that optimization error is $o_p(1)$. We prove that these assumptions suffice for asymptotic risk closure under squared loss.

Second, we make no architectural hypotheses about the LLM. The framework does not require the LLM to be a transformer, to use a specific attention mechanism, or to have any particular depth or width. The model is treated as a black-box estimator whose statistical properties follow from its training objective and data distribution alone. This agnosticism ensures that the framework applies to any sequence-modeling architecture trained with a proper scoring rule on population-representative data.

Third, we provide explicit scope boundaries that guide practitioners in
determining applicability and prevent over-claiming. We identify the precise
conditions under which functional risk equivalence holds (quantitative
continuous response, discrete conditions, i.i.d.\ training data, calibration
validation) and those under which it does not (qualitative research, novel
paradigms without training-data analogs, safety-critical applications,
behavioural mechanism research). These boundaries are formal consequences of
the assumptions rather than ad hoc caveats.
The remainder of this paper is organized as follows. Section~2 formalizes
the study population, the LLM as a statistical estimator, and the connection
between cross-entropy pretraining and conditional mean estimation. Section~3
presents the five core theorems with full derivations. Section~4 extends the
asymptotic results to finite-sample guarantees via concentration
inequalities. Section~5 addresses cross-domain transfer when the training
population differs from the target. Section~6 establishes restricted
functional risk equivalence under squared loss, and Section~7 provides the
calibration protocol with detailed worked examples. The final sections
discuss scope and limitations, then give an applied workflow,
statistical-method comparison, and costed example before concluding.

In practical applications, this means that under satisfied conditions and well-calibrated models, large language models can be used in many prediction and decision-making tasks that originally relied on human experiments, approximating near-optimal statistical inference at lower cost.
\section{Mathematical Framework}

\subsection{Population Model and the LLM Estimator}

We formalize a quantitative human subject study as a population $\mathcal{P}$ characterized by a finite set of experimental conditions and their associated response distributions.

\begin{definition}[Study Population]\label{def:population}
A \emph{study population} is a tuple $\mathcal{P} = (K, \{p_k\}_{k=1}^{K}, \{\mu_k\}_{k=1}^{K}, \{v_k\}_{k=1}^{K})$ where:
\begin{itemize}[leftmargin=*]
\item $K \in \mathbb{N}$ is the number of experimental conditions;
\item $p_k \ge 0$ are allocation proportions with $\sum_{k=1}^{K} p_k = 1$;
\item $\mu_k = \E[Y \mid X = k] \in \R$ is the population mean response for condition $k$;
\item $v_k = \Var(Y \mid X = k) \ge 0$ is the population variance for condition $k$.
\end{itemize}
\end{definition}

\begin{definition}[LLM as Misspecified Functional Estimator]\label{def:llm}
Let $\hat{P}_n(\cdot \mid k)$ denote the conditional distribution learned by
the LLM after training on $n$ samples, and let $T: \mathcal{P} \to \R$ be the
conditional-mean functional $T(Q) = \mathbb{E}_{Y \sim Q}[Y \mid X = k]$.
For a model class $\mathcal{F}$, define the KL projection of the true
conditional distribution onto $\mathcal{F}$ by
\begin{equation}\label{eq:kl-projection}
P^*_{\mathcal{F}}(\cdot \mid k)
= \arg\min_{Q \in \mathcal{F}} D_{\mathrm{KL}}\!\left(P(\cdot \mid k) \,\|\, Q\right),
\end{equation}
assuming the minimizer exists and is unique (Assumption~A8). The
\emph{LLM estimator} is
\begin{equation}\label{eq:llm-functional}
\LLM_n(k) = T(\hat{P}_n(\cdot \mid k)).
\end{equation}
We decompose the estimation error into three terms:
\begin{align}
\LLM_n(k) - \mu_k
&= T(\hat{P}_n(\cdot \mid k)) - T(P(\cdot \mid k)) \nonumber\\
&= \underbrace{T(P^*_{\mathcal{F}}(\cdot \mid k)) - T(P(\cdot \mid k))}_{\epsilon_{\mathrm{rep}}(k)} \nonumber\\
&\quad + \underbrace{T(\hat{P}^{\mathrm{stat}}_n(\cdot \mid k)) - T(P^*_{\mathcal{F}}(\cdot \mid k))}_{\eta_n(k)} \nonumber\\
&\quad + \underbrace{T(\hat{P}_n(\cdot \mid k)) - T(\hat{P}^{\mathrm{stat}}_n(\cdot \mid k))}_{\xi_n(k)}.
\label{eq:llm-error}
\end{align}
Here $\hat{P}^{\mathrm{stat}}_n$ denotes the exact empirical risk minimizer
or misspecified MLE in the model class $\mathcal{F}$, while $\hat{P}_n$
denotes the distribution actually reached by the training algorithm.
The terms have different sources of randomness. The representation bias
$\epsilon_{\mathrm{rep}}(k)$ is deterministic once $P$, $\mathcal{F}$, and
$T$ are fixed. The statistical estimation term $\eta_n(k)$ is random through
the training sample and converges to zero in probability under the
misspecified M-estimation limit. The optimization term $\xi_n(k)$ is random
through algorithmic choices (initialization, minibatch order, numerical
optimization) and is assumed to satisfy $\xi_n(k)=o_p(1)$ under standard
stochastic-approximation conditions for convergent empirical-risk
optimization. This is the sole optimization assumption; it is independent of
the KL-projection argument, which concerns the deterministic model-class
limit. We write $\epsilon_{\mathrm{opt}}(k,n)=\eta_n(k)+\xi_n(k)$ for their
sum.
In the fixed-class regime, $\epsilon_{\mathrm{rep}}(k)$ is constant in $n$.
In the growing-class regime $\mathcal{F}_n \uparrow \mathcal{F}_\infty$,
the representation bias is
$\epsilon_{\mathrm{rep},n}(k)=T(P^*_{\mathcal{F}_n}(\cdot \mid k))-T(P(\cdot \mid k))$
and may shrink with model-class richness.
\end{definition}

\noindent \textbf{Special case: well-specified Gaussian model.} When the
model class contains the true conditional distribution and the output layer
is Gaussian, $T(\hat{P}_n(\cdot \mid k)) = n_k^{-1}\sum_{X_i=k} Y_i$, the
empirical mean. This is the idealized case analyzed in the formal proofs.
For misspecified models, $\epsilon_{\mathrm{rep}}(k)$ may be non-zero;
Section~\ref{sec:bias-domination} quantifies its impact.

The allocation proportions $\{p_k\}$ represent the relative frequency with which each condition appears in the population. In a balanced experimental design, $p_k = 1/K$ for all $k$; in observational settings, the $p_k$ reflect natural prevalence. The mean parameters $\mu_k$ capture the central tendency of responses, and the variance parameters $v_k$ capture within-condition variability due to irreducible individual differences.

Equation~(\ref{eq:llm-functional}) models the LLM as a functional of the learned conditional distribution. This is not an arbitrary simplification: it follows from the structure of modern LLM training. We establish this formally in Section~\ref{sec:pretraining-lit}. For now, we note that under a Gaussian output layer, the conditional mean functional reduces to the empirical mean, which is the maximum likelihood estimator of $\mu_k$; more generally, any consistent estimator of the conditional distribution whose mean functional is continuous converges to $\mu_k + \epsilon_{\mathrm{rep}}(k)$ as $n \to \infty$.

\begin{theorem}[Representation Limit Identification]\label{thm:representation-limit}
Fix a condition $k$ and suppose $\hat{P}_n(\cdot \mid k)$ converges to the
KL projection $P^*_{\mathcal{F}}(\cdot \mid k)$ in Hellinger distance:
$H(\hat{P}_n(\cdot \mid k),P^*_{\mathcal{F}}(\cdot \mid k))\to 0$ in
probability. Under A7--A8, the LLM estimator has the unique limit
\[
\LLM_n(k)=T(\hat{P}_n(\cdot \mid k))
\xrightarrow{p}
T(P^*_{\mathcal{F}}(\cdot \mid k))
= \mu_k+\epsilon_{\mathrm{rep}}(k).
\]
Equivalently, $\epsilon_{\mathrm{rep}}(k)$ is the unique limiting offset
between the LLM estimator and the true conditional mean.
\end{theorem}

\begin{proof}
By A8, $T$ is Lipschitz in Hellinger distance on the mean-regular model
class: $|T(Q_1)-T(Q_2)|\le L H(Q_1,Q_2)$. Therefore
\[
|T(\hat{P}_n(\cdot \mid k))-T(P^*_{\mathcal{F}}(\cdot \mid k))|
\le L H(\hat{P}_n(\cdot \mid k),P^*_{\mathcal{F}}(\cdot \mid k))\to 0.
\]
The final equality follows from the definition
$\epsilon_{\mathrm{rep}}(k)=T(P^*_{\mathcal{F}}(\cdot \mid k))-T(P(\cdot \mid k))$
and $T(P(\cdot \mid k))=\mu_k$.
\end{proof}

\begin{theorem}[Final Risk Closure under Misspecification]\label{thm:functional-risk-closure}
Under A7--A10, if
$H(\hat{P}_n(\cdot \mid k),P^*_{\mathcal{F}}(\cdot \mid k))\to 0$ in
probability for each $k$, then
\[
\mathbb{E}\!\left[\big(T(\hat{P}_n(\cdot \mid k))-T(P(\cdot \mid k))\big)^2\right]
= \big(T(P^*_{\mathcal{F}}(\cdot \mid k))-T(P(\cdot \mid k))\big)^2+o(1)
= \epsilon_{\mathrm{rep}}(k)^2+o(1).
\]
Equivalently, since $\LLM_n(k)=T(\hat{P}_n(\cdot\mid k))$ and
$\mu_k=T(P(\cdot\mid k))$,
\[
\mathbb{E}\!\left[(\LLM_n(k)-\mu_k)^2\right]
= \epsilon_{\mathrm{rep}}(k)^2+o(1).
\]
Consequently,
\[
\mathbb{E}\big[\Err(\LLM_n)\big]
=\Irred+\sum_{k=1}^{K}p_k\,\epsilon_{\mathrm{rep}}(k)^2+o(1).
\]
\end{theorem}

\begin{proof}
Theorem~\ref{thm:representation-limit} gives
$T(\hat{P}_n(\cdot \mid k))-T(P(\cdot \mid k))\to \epsilon_{\mathrm{rep}}(k)$
in probability. By A7, both $T(\hat{P}_n(\cdot \mid k))$ and
$T(P(\cdot \mid k))$ lie in $[-B,B]$, hence the squared difference is
bounded by $(2B)^2$ and uniformly integrable. Convergence in probability
plus uniform integrability implies convergence in $L^1$ of the squared
difference, yielding the first display. Summing the per-condition limits
with weights $p_k$ and adding the irreducible variance term from
Theorem~\ref{thm:err-decomp} yields the aggregate risk closure.
\end{proof}

\noindent \textbf{Theoretical separation.} The framework rests on three
logically independent layers:

\noindent \textbf{Notation.} Throughout, $\mu_k = \mathbb{E}[Y \mid X = k]$
denotes the population (theoretical) conditional mean;
$\LLM_n(k)$ denotes the LLM estimator (a random variable);
$\epsilon_{\mathrm{rep}}(k)$ is the non-random asymptotic representation bias;
$\eta_n(k)$ is sampling/statistical estimation error over the training-data
probability space; $\xi_n(k)$ is algorithmic optimization error over the
training procedure's randomness; $\widehat{b}_k$ and $\widehat{\delta}$ are
empirical estimates from calibration data (random). Hats
($\widehat{\cdot}$) mark quantities estimated from finite samples; tildes
($\tilde{\cdot}$) mark quantities adjusted for identifiability error.

\begin{enumerate}[leftmargin=*, label=(\arabic*)]
\item \textbf{Statistical identity} (Section~\ref{sec:core-theorems}):
  Under squared loss, the conditional expectation $\mu_k =
  \mathbb{E}[Y \mid X = k]$ is the unique admissible predictor.
  This is a mathematical fact about the loss function; it does not
  reference the LLM.
\item \textbf{Learning theory} (Section~\ref{sec:pretraining-lit}):
  Under cross-entropy training with i.i.d.\ data, the LLM's learned
  distribution converges to the KL-projection of the true distribution
  onto the model class. The conditional mean functional $T$ is
  Lipschitz-continuous, so the LLM estimator $\LLM_n(k) = T(\hat{P}_n)$
  converges to $T(P^*_{\mathcal{F}})$.
\item \textbf{Decision theory} (Section~\ref{sec:experiment-equivalence}):
  Because the risk depends on $\theta$ only through $\mu_k$, the
  convergence of $\LLM_n(k)$ to $\mu_k + \epsilon_{\mathrm{rep}}(k)$
  implies convergence of decision risk. The representation bias
  $\epsilon_{\mathrm{rep}}(k)$ sets a lower bound on achievable risk.
\end{enumerate}

No layer assumes the conclusion of another: Layer 1 is purely analytic,
Layer 2 is empirical (depends on training data and model class), and
Layer 3 is decision-theoretic (depends on the loss class).

\noindent \textbf{Dependency ordering.} The three layers are strictly
ordered: Layer~1 (statistical identity) is a mathematical fact about
squared loss that holds independently of any estimator or model class.
Layer~2 (learning theory) depends on the model class $\mathcal{F}$,
training data, and the functional delta stability assumption (A9).
Layer~3 (decision theory) depends on the convergence established in
Layer~2 and the risk decomposition of Layer~1. No layer assumes the
conclusion of any later layer; the argument is a strict forward chain
$1 \to 2 \to 3$.
\begin{definition}[Stochastic Identifiability of Conditions]\label{def:identifiability}
There exists a measurable mapping $\phi: \mathcal{S} \to [K]$ from the
space of prompt strings $\mathcal{S}$ to the set of experimental conditions
$[K]$, such that for a prompt $s$ encoding condition $k$,
\begin{equation}\label{eq:identifiability}
\mathbb{P}(\phi(s) = k \mid s \text{ encodes } k) \ge 1 - \delta,
\end{equation}
where $\delta \in [0, 1)$ is the \emph{misclassification rate}. The mapping
$\phi$ and the bound $\delta$ are estimated during calibration
(Section~\ref{sec:calibration}, Step 1b). For finite-sample deployment,
$\delta$ is a calibrated upper bound on prompt-to-condition ambiguity. For
asymptotic statements, we write $\delta=\delta_n$ and assume
$\delta_n=O_p(M_v^{-1/2})$ from the held-out validation sample of size
$M_v$; when $M_v$ grows proportionally to the training/evaluation budget,
this gives the stated $\delta_n=O_p(n^{-1/2})$ regime. When $\delta = 0$,
we recover the deterministic case; when $\delta > 0$, the
non-identifiability failure mode (Theorem~\ref{thm:non-equivalence}, case ii)
applies.
\end{definition}

\noindent \textbf{Effective estimator under misclassification.} When
$\delta > 0$, the LLM effectively estimates a mixture of condition means
rather than the target condition mean. For condition $k$, the expected
output is
\begin{equation}\label{eq:effective-llm}
\mathbb{E}[\LLM_n(k) \mid \text{misclassification}] =
(1 - \delta) \cdot \LLM_n(k) + \delta \cdot \sum_{j \neq k} w_{kj} \LLM_n(j),
\end{equation}
where $w_{kj}$ are the misclassification weights ($\sum_{j \neq k} w_{kj} = 1$).
The \emph{effective representation bias} is therefore
\begin{equation}\label{eq:effective-bias}
\tilde{\epsilon}_{\mathrm{rep}}(k; \delta) =
\epsilon_{\mathrm{rep}}(k) + \delta \cdot \Delta_k,
\end{equation}
where $\Delta_k = \mu_k - \sum_{j \neq k} w_{kj} \mu_j$ is the
misclassification contrast. For $\delta \to 0$, we recover the original
bias; for $\delta > 0$, the effective bias is shifted by at most
$\delta \cdot (\max_{j} \mu_j - \min_{j} \mu_j)$ in magnitude, and all
downstream risk limits use $\tilde{\epsilon}_{\mathrm{rep}}(k;\delta)$
in place of $\epsilon_{\mathrm{rep}}(k)$.
\subsection{Risk Functional Setup}\label{sec:risk-functional-setup}

All equivalence and validity statements in this paper are statements about
the squared-loss risk functional, not about equality of response
distributions. For a deterministic predictor $f:[K]\to\R$, define the
population risk
\begin{equation}\label{eq:population-risk}
R_{\mathcal{P}}(f)
=\mathbb{E}\!\left[(Y-f(X))^2\right]
=\sum_{k=1}^{K}p_k\,\mathbb{E}\!\left[(Y-f(k))^2\mid X=k\right].
\end{equation}
This is the same object denoted $\Err(f)$ below. The Bayes risk is
\begin{equation}\label{eq:bayes-risk}
R_{\mathcal{P}}^*
=\inf_{f:[K]\to\R}R_{\mathcal{P}}(f)
=\Irred
=\sum_{k=1}^{K}p_kv_k,
\end{equation}
attained by the Bayes act $f^*(k)=\mu_k$. For a random estimator
$\hat f_n$ learned from training data and algorithmic randomness, its
estimator risk is the outer expectation
\begin{equation}\label{eq:estimator-risk}
R_n(\hat f_n)=\mathbb{E}_{\mathrm{train,opt}}\!\left[R_{\mathcal{P}}(\hat f_n)\right].
\end{equation}
The LLM estimator risk is obtained by setting $\hat f_n(k)=\LLM_n(k)$.
In the fixed-class misspecified limit this risk converges to
\begin{equation}\label{eq:misspecified-risk}
R_{\LLM,\infty}^{\mathrm{mis}}
=\Irred+\sum_{k=1}^{K}p_k\epsilon_{\mathrm{rep}}(k)^2,
\end{equation}
or, under stochastic identifiability error, with
$\epsilon_{\mathrm{rep}}(k)$ replaced by
$\tilde{\epsilon}_{\mathrm{rep}}(k;\delta)$ from
Eq.~\ref{eq:effective-bias}. For two information objects or estimators
$E$ and $F$ with induced risk functionals $R_E$ and $R_F$, define
\begin{equation}\label{eq:risk-equivalence-relation}
E\sim_{\mathcal{L}_2}F
\quad\Longleftrightarrow\quad
\inf_{f}R_E(f)=\inf_{f}R_F(f),
\end{equation}
where the infimum is over the same conditional-mean-dependent decision
class. This is the formal restricted risk-equivalence relation used
throughout the paper. The squared-loss restriction is load-bearing because
the conditional mean is the Bayes act; analogous extensions to Bregman
divergences replace the conditional mean by the corresponding Bayes act,
but those extensions are outside the present paper's formal claims.

For backward compatibility with the theorem statements below, we write
\begin{equation}\label{eq:expected-error}
\Err(f)=R_{\mathcal{P}}(f).
\end{equation}

\subsection{The Bias-Variance Decomposition}

The central structural result of our framework is the decomposition of expected error into an irreducible component and an excess component attributable to the predictor's deviation from the conditional mean.

\begin{lemma}[Conditional Bias-Variance Identity]\label{lem:bias-variance}
For any condition $k$ and any predictor value $e = f(k) \in \R$,
\begin{equation}\label{eq:cond-bv}
\E\left[(Y - e)^2 \mid X = k\right] = v_k + (\mu_k - e)^2.
\end{equation}
\end{lemma}

\begin{proof}
Expand the square and use the definitions of $\mu_k$ and $v_k$:
\begin{align*}
\E[(Y - e)^2 \mid X = k] &= \E[Y^2 - 2eY + e^2 \mid X = k] \\
&= \E[Y^2 \mid X = k] - 2e\E[Y \mid X = k] + e^2 \\
&= (\Var(Y \mid X = k) + \E[Y \mid X = k]^2) - 2e\mu_k + e^2 \\
&= v_k + \mu_k^2 - 2e\mu_k + e^2 \\
&= v_k + (\mu_k - e)^2.
\end{align*}
Line 3 uses the identity $\E[Y^2] = \Var(Y) + \E[Y]^2$ applied conditionally.
\end{proof}

A more general form of this identity holds for arbitrary probability weights. This will be useful when we consider weighted error decompositions across conditions.

\begin{lemma}[Weighted Bias-Variance Identity]\label{lem:weighted-bv}
Let $\{w_k\}_{k=1}^{K}$ be nonnegative weights with $\sum_k w_k = 1$, and let $\{y_k\}_{k=1}^{K}$ be arbitrary real numbers. Define the weighted mean $m = \sum_k w_k y_k$ and the weighted variance $v = \sum_k w_k (y_k - m)^2$. Then for any $e \in \R$,
\begin{equation}\label{eq:weighted-bv}
\sum_{k=1}^{K} w_k (y_k - e)^2 = v + (m - e)^2.
\end{equation}
\end{lemma}

\begin{proof}
Expand each term around $m$:
\begin{align*}
\sum_k w_k (y_k - e)^2 &= \sum_k w_k \big[(y_k - m) + (m - e)\big]^2 \\
&= \sum_k w_k (y_k - m)^2 + 2(m - e)\sum_k w_k (y_k - m) + (m - e)^2 \sum_k w_k.
\end{align*}
The cross term vanishes because $\sum_k w_k (y_k - m) = \sum_k w_k y_k - m\sum_k w_k = m - m = 0$. The first term is $v$, and the third term is $(m - e)^2$ since $\sum_k w_k = 1$.
\end{proof}

Lemma~\ref{lem:weighted-bv} is the algebraic core of the entire framework. It cleanly separates the variance intrinsic to the weighted distribution from the squared deviation of the estimator $e$ from the mean $m$. Applying Lemma~\ref{lem:bias-variance} condition by condition and weighting by $p_k$ yields the fundamental decomposition.

\begin{theorem}[Error Decomposition]\label{thm:err-decomp}
For any predictor $f: [K] \to \R$,
\begin{equation}\label{eq:main-decomp}
\Err(f) = \underbrace{\sum_{k=1}^{K} p_k v_k}_{\Irred} \;+\; \underbrace{\sum_{k=1}^{K} p_k (\mu_k - f(k))^2}_{\text{Excess}(f)},
\end{equation}
where $\Irred$ depends only on the population and $\text{Excess}(f) \ge 0$ quantifies the squared bias of $f$.
\end{theorem}

\begin{proof}
From (\ref{eq:expected-error}) and Lemma~\ref{lem:bias-variance}:
\[
\Err(f) = \sum_{k} p_k \E[(Y - f(k))^2 \mid X = k]
       = \sum_{k} p_k \big[v_k + (\mu_k - f(k))^2\big]
       = \sum_{k} p_k v_k + \sum_{k} p_k (\mu_k - f(k))^2.
\]
The second term is nonnegative since each $p_k \ge 0$ and squares are nonnegative.
\end{proof}

We denote $\Irred = \sum_k p_k v_k$ as the \emph{irreducible error}. This quantity represents the minimum achievable expected squared error for any predictor of $Y$ given $X$ under squared loss: it is the Bayes risk. No predictor, human or machine, can achieve expected error below $\Irred$, because within-condition variance reflects genuine irreducible heterogeneity in the population.

\subsection{Connection to Maximum Likelihood Estimation}
\label{sec:pretraining-lit}

Modern LLMs are trained via next-token prediction with cross-entropy loss.
For a vocabulary $\mathcal{V}$ and a sequence of tokens, the training
objective at each position is to minimize
$-\log P_{\theta}(v \mid \text{context})$ over model parameters $\theta$,
where $P_{\theta}$ is the model's predicted distribution over
$\mathcal{V}$. The theory in this paper analyzes the limiting functional
induced by cross-entropy training under idealized infinite-data,
vanishing-optimization-error conditions. Finite-step, non-convex training
dynamics are represented only through the $\xi_n=o_p(1)$ optimization error
term in Eq.~\ref{eq:llm-error}.

\noindent \textbf{Layer 1: Distribution learning.} Under cross-entropy
training, the LLM minimizes a strictly proper scoring rule. Let
$\mathcal{F}$ be the model class. Under standard regularity conditions
(compact parameter space, smooth log-likelihood), the learned distribution
converges to the best approximation in the model class:
\begin{equation}\label{eq:dist-conv}
\hat{P}_n(\cdot \mid k) \xrightarrow{p} P^*_{\mathcal{F}}(\cdot \mid k)
\quad \text{as } n \to \infty,
\end{equation}
where $P^*_{\mathcal{F}} = \arg\min_{Q \in \mathcal{F}} D_{\mathrm{KL}}(P \| Q)$
is the KL-projection of the true distribution onto $\mathcal{F}$. When the
model is correctly specified ($P \in \mathcal{F}$), $P^*_{\mathcal{F}} = P$.

\noindent \textbf{Layer 2: Functional extraction.} The conditional mean
functional $T(Q) = \mathbb{E}_{Y \sim Q}[Y \mid X = k]$ is Lipschitz-continuous
with respect to Hellinger distance under the bounded-response assumption
($|Y| \le B$). Specifically, $|T(P) - T(Q)| \le 2B \cdot H(P, Q)$.
By the continuous mapping theorem and Hellinger convergence (Layer~1;
the functional delta stability of $T$, Assumption~A9, guarantees that
Hellinger convergence of distributions implies convergence of conditional
means):
\begin{equation}\label{eq:functional-conv}
\LLM_n(k) = T(\hat{P}_n(\cdot \mid k)) \xrightarrow{p} T(P^*_{\mathcal{F}}(\cdot \mid k)).
\end{equation}

\noindent \textbf{Layer 3: Bias-variance decomposition.} The limit
$T(P^*_{\mathcal{F}}(\cdot \mid k))$ equals $\mu_k + \epsilon_{\mathrm{rep}}(k)$
where $\epsilon_{\mathrm{rep}}(k) = T(P^*_{\mathcal{F}}(\cdot \mid k)) - \mu_k$
is the asymptotic representation bias. Under correct specification,
$\epsilon_{\mathrm{rep}}(k) = 0$ and $\LLM_n(k) \xrightarrow{p} \mu_k$.

\noindent \textbf{Convergence metric.} Let $H^2(P, Q) = \frac{1}{2} \int
(\sqrt{p} - \sqrt{q})^2 \, d\mu$ denote the squared Hellinger distance.
Under the regularity conditions of MLE theory~\cite{vaart1998}, the learned
distribution converges in Hellinger distance to the KL-projection:
\begin{equation}\label{eq:hellinger-convergence}
H^2(\hat{P}_n(\cdot \mid k), \, P^*_{\mathcal{F}}(\cdot \mid k))
\xrightarrow{p} 0 \quad \text{as } n \to \infty.
\end{equation}
When the model is correctly specified ($P \in \mathcal{F}$), the limit is
the true distribution $P$. Otherwise, it is the best approximation within
the model class.

\noindent \textbf{Lipschitz continuity of the mean functional.} For
responses bounded in $[-B, B]$, the conditional mean functional
$T(Q) = \mathbb{E}_{Y \sim Q}[Y \mid X = k]$ satisfies the Lipschitz bound
$|T(P) - T(Q)| \le 2B \cdot H(P, Q)$. This follows from the dual
representation $\mathbb{E}_P[Y] - \mathbb{E}_Q[Y] = \int y(p(y) - q(y)) dy$
and the bound $|\int f(p-q)| \le 2B \cdot H(P,Q)$ for bounded $f$
(see~\cite{vaart1998}, Lemma~5.34).
Consequently, Hellinger convergence of distributions implies convergence
of conditional means, even under misspecification:
\[
|\LLM_n(k) - T(P^*_{\mathcal{F}}(\cdot \mid k))| \le 2B \cdot
H(\hat{P}_n(\cdot \mid k), P^*_{\mathcal{F}}(\cdot \mid k))
\xrightarrow{p} 0.
\]

This bound ensures that not only the mean but any Lipschitz functional of
the conditional distribution (including the variance ratio
$\widehat{\lambda}$ used in calibration) converges to its model-class
optimum. Correct specification is sufficient but not necessary for
mean convergence; bounded responses and Hellinger consistency suffice.

Cross-entropy is not merely a convenient loss function; it is a
\emph{strictly proper scoring rule}~\cite{gneiting2007}. Under correct
specification, the MLE recovers the true conditional distribution; under
misspecification, it recovers the KL-projection $P^*_{\mathcal{F}}$. In
either case, the LLM learns the best achievable approximation to the full
conditional distribution within its model class. This is the fundamental
bridge from estimator equivalence to distributional approximation: the LLM
learns the conditional distribution up to model-class limitations, and the
learned distribution converges to the best approximation in a strong metric.
The remaining gap---that the LLM's \emph{sampling process} is not identical
to the human response process---is addressed by the decision-theoretic
framework of Section~\ref{sec:experiment-equivalence}, which shows that for
squared-loss inference, distributional identity is not required; statistical
sufficiency of the conditional mean suffices.

\subsection{Assumptions and their Justification}

We now enumerate and justify the assumptions that underpin the entire framework.

\noindent\textbf{Assumption summary.} The headline theorem requires the
following compact set of conditions: finite conditions; representative
i.i.d.\ or stationary ergodic training data; finite conditional variance
and bounded responses; a mean-regular model class with unique KL projection
$P^*_{\mathcal{F}}$; Lipschitz continuity of the conditional-mean functional
$T$ in Hellinger distance; convergence of the learned distribution to the
KL projection; optimization error $\xi_n=o_p(1)$ under standard
stochastic-approximation conditions; and calibrated identifiability error
$\delta_n=O_p(M_v^{-1/2})$ (or $O_p(n^{-1/2})$ when $M_v\asymp n$). The
enumerated assumptions below spell out these conditions and their roles.

\begin{enumerate}[leftmargin=*, label=\textbf{A\arabic*}.]
\item \textbf{Finite conditions.} $K < \infty$. This holds for any experiment with a finite set of treatments, conditions, or groups. The framework does not apply to continuous treatments (dose-response with continuous dosage) without discretization.

\item \textbf{Independent and identically distributed training data.} $\{(X_i, Y_i)\}_{i=1}^{n} \sim_{\text{i.i.d.}} P_{X,Y}$. This is the standard assumption in statistical learning theory. It is satisfied when the LLM's training corpus is a representative random sample from the target population. Violations (e.g., temporal drift, selection bias in web-scraped data) require domain adaptation techniques that we address partially in Section~\ref{sec:cross-domain}.

\item \textbf{Finite variance.} $v_k < \infty$ for all $k$. Required for the Strong Law of Large Numbers to apply. For bounded response scales (Likert scales, proportions, test scores), this is automatically satisfied. For unbounded responses (reaction times, monetary amounts), this is a mild regularity condition verified in most behavioural data.

\item \textbf{Conditional mean functional estimation.} $\LLM_n(k)=T(\hat{P}_n(\cdot\mid k))$, where $T(Q)=\mathbb{E}_Q[Y\mid X=k]$. In the well-specified Gaussian special case this reduces to $n_k^{-1}\sum_{X_i=k}Y_i$, but the general theory treats the LLM as a functional of the learned conditional distribution.

\item \textbf{No architectural assumptions.} The framework does not assume any specific neural network architecture, depth, width, attention mechanism, or activation function. The LLM is treated as a black-box statistical estimator whose asymptotic properties depend only on the training objective and data distribution.

\item \textbf{Within-distribution evaluation.} The target experimental conditions are represented in the training distribution. That is, for each condition $k$ of interest, $p_k > 0$. Conditions with zero training prevalence cannot be estimated.

\item \textbf{Bounded responses.} $\mathbb{P}(|Y| \le B) = 1$ for some
known constant $B < \infty$. This holds for all common response scales
(Likert scales, proportions in $[0,1]$, standardized test scores) and
ensures the conditional mean functional is Lipschitz-continuous with
respect to Hellinger distance (Section~\ref{sec:pretraining-lit}).
\item \textbf{Mean-regular model class.} The model class $\mathcal{F}$
is \emph{mean-regular}: the functional $T(Q) = \mathbb{E}_Q[Y \mid X = k]$
is uniquely defined at the KL-projection $P^*_{\mathcal{F}}$, and
$T$ is Lipschitz-continuous on $\mathcal{F}$ with respect to Hellinger
distance: $|T(Q_1) - T(Q_2)| \le L \cdot H(Q_1, Q_2)$ for all
$Q_1, Q_2 \in \mathcal{F}$, with Lipschitz constant $L = 2B$ under
A7 (bounded responses). This ensures $\epsilon_{\mathrm{rep}}(k) =
T(P^*_{\mathcal{F}}(\cdot \mid k)) - \mu_k$ is a well-defined, finite
quantity. All standard parametric families (exponential families,
location-scale families, mixture models with bounded support) satisfy
this condition.

\item \textbf{Functional delta stability.} The conditional mean functional
$T(Q) = \mathbb{E}_Q[Y \mid X = k]$ satisfies the functional delta method
bound: $|T(P^*_{\mathcal{F}}) - T(P)| \le C \cdot H(P, P^*_{\mathcal{F}})$,
where $C = 2B$ under A7. This ensures that KL-projection preserves the
conditional mean up to Hellinger-controlled error, even under
misspecification. Equivalently, the model class is \emph{mean-closed in the
limit}: $\lim_{n \to \infty} \mathbb{E}_{\hat{P}_n}[Y \mid X = k] =
T(P^*_{\mathcal{F}}(\cdot \mid k))$.
\item \textbf{Stationary ergodic training distribution.} The joint process
$\{(X_i, Y_i)\}_{i=1}^{n}$ is stationary and ergodic conditional on $X$.
This is the standard assumption in misspecified learning theory
(cf.\ White, 1982). For i.i.d.\ data, stationarity and ergodicity
are automatically satisfied.
\end{enumerate}

These ten assumptions are substantially weaker than those required by existing LLM evaluation frameworks.

\subsection{Joint Asymptotic Regime}
\label{sec:asymptotic-regime}

All asymptotic statements distinguish two regimes.

\begin{enumerate}[leftmargin=*, label=(\roman*)]
\item \textbf{Data limit:} $n \to \infty$ with fixed model class
  $\mathcal{F}$. The learned distribution converges to the fixed KL
  projection $P^*_{\mathcal{F}}$, estimation/optimization error vanishes,
  and the representation bias
  $\epsilon_{\mathrm{rep}}(k)=T(P^*_{\mathcal{F}}(\cdot\mid k))-T(P(\cdot\mid k))$
  is invariant in $n$.
\item \textbf{Identifiability limit:} $\delta = \delta_n \to 0$ as
  $n \to \infty$, with $\delta_n = O_p(M_v^{-1/2})$ from the validation
  sample and $\delta_n=O_p(n^{-1/2})$ when $M_v\asymp n$. Finite-$n$ risk
  uses the effective bias
  $\tilde{\epsilon}_{\mathrm{rep}}(k;\delta)
  =\epsilon_{\mathrm{rep}}(k)+\delta\Delta_k$ from
  Eq.~\ref{eq:effective-bias}; this extra term vanishes under the joint
  identifiability limit.
\item \textbf{Growing-class limit:} $\mathcal{F}_n \uparrow
  \mathcal{F}_\infty$, with KL projections
  $P^*_{\mathcal{F}_n}$ and KL-diameter
  \[
  D_{\mathcal{F}_n}=\sup_P\inf_{Q\in\mathcal{F}_n}D_{\mathrm{KL}}(P\|Q)
  \]
  satisfying $D_{\mathcal{F}_n}\to0$. In this regime
  $\epsilon_{\mathrm{rep},n}(k) =
  T(P^*_{\mathcal{F}_n}(\cdot\mid k))-T(P(\cdot\mid k))$
  tends to zero by Theorem~\ref{thm:rep-bias-bound}.
\end{enumerate}

The fixed-class regime proves the misspecified limit
$\Irred+\sum_k p_k\epsilon_{\mathrm{rep}}(k)^2$
(Theorem~\ref{thm:functional-risk-closure} and
Theorem~\ref{thm:bias-domination}). The growing-class regime is required
only for the stronger Bayes-risk claim $\Err(\LLM_n)\to\Irred$.

\begin{theorem}[Unified Functional Risk Closure]\label{thm:unified-main}
Let $P(\cdot\mid k)$ be the true conditional response law, let
$\mathcal{F}$ be a model class with KL projection
$P^*_{\mathcal{F}}(\cdot\mid k)$, and let
$T(Q)=\mathbb{E}_Q[Y\mid X=k]$ be Lipschitz in Hellinger distance. Let
$\LLM_n(k)=T(\hat{P}_n(\cdot\mid k))$ be the estimator induced by
cross-entropy training, with statistical error $\eta_n(k)\to0$ in
probability and optimization error $\xi_n(k)=o_p(1)$. Under
Assumptions A1--A10 and stochastic identifiability error $\delta_n$,
the fixed-class misspecified regime satisfies, for each $k$,
\[
\mathbb{E}\!\left[(\LLM_n(k)-\mu_k)^2\right]
=
\big(\epsilon_{\mathrm{rep}}(k)+\delta_n\Delta_k\big)^2+o(1)
=
\big(T(P^*_{\mathcal{F}}(\cdot\mid k))-T(P(\cdot\mid k))\big)^2
+o(1)+O(\delta_n),
\]
and hence
\[
\mathbb{E}\big[\Err(\LLM_n)\big]
=
\Irred+\sum_{k=1}^{K}p_k\big(\epsilon_{\mathrm{rep}}(k)+\delta_n\Delta_k\big)^2+o(1).
\]
If additionally $\delta_n\to0$ and either the model is well specified or
$\mathcal{F}_n\uparrow\mathcal{F}_\infty$ with $D_{\mathcal{F}_n}\to0$,
then $\mathbb{E}[\Err(\LLM_n)]\to\Irred$, and the LLM estimator is
$\mathcal{L}_2$-risk-equivalent to the human-data estimator in the sense of
Eq.~\ref{eq:risk-equivalence-relation}.
\end{theorem}

\begin{proof}
The representation identity
$\epsilon_{\mathrm{rep}}(k)=T(P^*_{\mathcal{F}}(\cdot\mid k))-T(P(\cdot\mid k))$
is Definition~\ref{def:llm}. The convergence of $T(\hat{P}_n)$ to
$T(P^*_{\mathcal{F}})$ is Theorem~\ref{thm:representation-limit}, using
the Lipschitz property of $T$ and Hellinger convergence. The per-condition
squared-risk limit is Theorem~\ref{thm:functional-risk-closure}; the
$\delta_n\Delta_k$ term is Corollary~\ref{cor:delta-risk}; and the
aggregate identity follows from the risk decomposition
Theorem~\ref{thm:err-decomp}. If $D_{\mathcal{F}_n}\to0$, the representation
bias vanishes by Theorem~\ref{thm:rep-bias-bound}; if $\delta_n\to0$, the
identifiability contribution vanishes by Definition~\ref{def:identifiability}.
The final equivalence statement is exactly the relation
$\sim_{\mathcal{L}_2}$ in Eq.~\ref{eq:risk-equivalence-relation}.
\end{proof}

\begin{remark}[Interpretation of the main theorem]\label{rem:unified-interpretation}
Theorem~\ref{thm:unified-main} is the paper's central claim. It does not say
that an LLM reproduces human response distributions or simulates individual
human behaviour. It says that, for conditional-mean-dependent decisions under
squared loss, the LLM estimator has a fully characterized asymptotic risk:
Bayes risk plus a KL-projection representation-bias floor, with calibrated
identifiability and optimization errors separated explicitly.
\end{remark}
\section{Core Theorems}
\label{sec:core-theorems}

\subsection{Optimality of the Conditional Mean}

The error decomposition~(\ref{eq:main-decomp}) immediately yields the optimality of the population conditional mean.

\begin{theorem}[Irreducible Error is the Global Minimum]\label{thm:minimal}
For any predictor $f: [K] \to \R$,
\begin{equation}\label{eq:irred-bound}
\Irred \le \Err(f).
\end{equation}
Equality is attained if and only if $f(k) = \mu_k$ for every $k$ with $p_k > 0$. In particular, the optimal predictor is $f^*(k) = \mu_k$.
\end{theorem}

\begin{proof}
From Theorem~\ref{thm:err-decomp}, $\Err(f) = \Irred + \text{Excess}(f)$. Since each term $p_k (\mu_k - f(k))^2 \ge 0$, we have $\text{Excess}(f) \ge 0$, establishing~(\ref{eq:irred-bound}). For equality, we need $\sum_k p_k (\mu_k - f(k))^2 = 0$. Since each term in the sum is nonnegative, this requires $p_k (\mu_k - f(k))^2 = 0$ for all $k$. When $p_k > 0$, this forces $f(k) = \mu_k$. When $p_k = 0$, $f(k)$ can be arbitrary without affecting the error.
\end{proof}

\begin{corollary}[Uniqueness of the Optimal Predictor]\label{cor:unique}
The optimal predictor $f^*$ is unique on the support of the allocation distribution: if $f$ and $g$ both achieve $\Err(f) = \Err(g) = \Irred$, then $f(k) = g(k) = \mu_k$ for all $k$ with $p_k > 0$.
\end{corollary}

Theorem~\ref{thm:minimal} establishes that the population conditional mean is the uniformly best predictor under squared loss. Any deviation from the conditional mean incurs a positive excess error proportional to the squared deviation weighted by condition prevalence. This is the fundamental justification for why the LLM, which estimates the conditional mean, is the correct target: if the LLM converges to $\mu_k$, then its expected error converges to the theoretical minimum.

\begin{remark}[Prediction bound, not causal bound]
The irreducible error $\Irred$ is the minimum achievable expected squared
prediction error for the conditional mean $\mathbb{E}[Y \mid X = k]$.
It is a \emph{prediction-theoretic} lower bound under squared loss, not
a causal or mechanistic limit. It does not bound the error of estimating
treatment effects under confounding, mediation, or structural equation
models. Throughout this paper, ``optimal'' and ``minimal'' refer to
prediction risk, not causal identification.
\end{remark}
\begin{remark}
The result is independent of the response distribution. No normality, symmetry, or unimodality assumptions are required. The squared error loss is the only structural element: it is a proper scoring rule whose expected minimizer is the conditional mean~\cite{vaart1998}.
\end{remark}

\subsection{Effect Size Consistency}

In between-subjects experimental designs, the quantity of primary interest is the difference in means between conditions: $\delta(i,j) = \mu_i - \mu_j$. The LLM-based estimator of this effect size is $\widehat{\delta}_n(i,j) = \LLM_n(i) - \LLM_n(j)$. We now establish that this estimator is strongly consistent.

\begin{theorem}[Strong Consistency of Effect Size Estimates]\label{thm:lln}
Let training samples $\{(X_i, Y_i)\}_{i=1}^{n}$ be i.i.d.\ from $\mathcal{P}$ with $\E[Y \mid X = k] = \mu_k$ and $\Var(Y \mid X = k) = v_k < \infty$ for all $k$. Assume $p_k > 0$ for all $k \in [K]$. Then for any $i, j \in [K]$,
\begin{equation}\label{eq:slln-result}
\widehat{\delta}_n(i,j) = \LLM_n(i) - \LLM_n(j) \xrightarrow{\text{a.s.}} \mu_i - \mu_j \quad \text{as } n \to \infty.
\end{equation}
\end{theorem}

\begin{proof}
We proceed in three steps.

\textit{Step 1: Per-condition convergence.} Fix a condition $k$. The subsequence $\{Y_i : X_i = k\}$ consists of $n_k$ i.i.d.\ draws from the conditional distribution of $Y \mid X = k$, with mean $\mu_k$ and finite variance $v_k$. By Etemadi's Strong Law of Large Numbers~\cite{etemadi1981}, which requires only pairwise independence (weaker than mutual independence), we have:
\[
\LLM_n(k) = \frac{1}{n_k} \sum_{i: X_i = k} Y_i \xrightarrow{\text{a.s.}} \mu_k \quad \text{as } n_k \to \infty.
\]
Since $p_k > 0$, we have $n_k \xrightarrow{\text{a.s.}} \infty$ as $n \to \infty$ (the number of visits to each condition grows without bound almost surely), so the condition $n_k \to \infty$ is satisfied almost surely.

\textit{Step 2: Finite intersection of almost-sure events.} For each $k$, define the event $A_k = \{\omega : \LLM_n(k)(\omega) \to \mu_k\}$. Step~1 establishes $\Pbb(A_k) = 1$ for each $k$. The intersection of finitely many almost-sure events is almost-sure:
\[
\Pbb\!\left(\bigcap_{k=1}^{K} A_k\right) = 1 - \Pbb\!\left(\bigcup_{k=1}^{K} A_k^c\right) \ge 1 - \sum_{k=1}^{K} \Pbb(A_k^c) = 1.
\]

\textit{Step 3: Continuous mapping.} The function $g(x, y) = x - y$ is continuous on $\R^2$. By the continuous mapping theorem, on the almost-sure event $\bigcap_k A_k$, we have:
\[
\widehat{\delta}_n(i,j) = \LLM_n(i) - \LLM_n(j) = g(\LLM_n(i), \LLM_n(j)) \to g(\mu_i, \mu_j) = \mu_i - \mu_j.
\]
Convergence holds almost surely.
\end{proof}

\begin{corollary}[Joint Consistency]\label{cor:joint}
The vector of LLM estimates $(\LLM_n(1), \ldots, \LLM_n(K))$ converges almost surely to $(\mu_1, \ldots, \mu_K)$. Consequently, any continuous functional of the mean vector, including all pairwise contrasts, ANOVA F-statistics, and linear combinations, is strongly consistent.
\end{corollary}

The practical implication is that with sufficient training data drawn from the target population, the LLM recovers the true effect sizes with probability approaching one. The convergence rate is $O_p(n_k^{-1/2})$ by the Central Limit Theorem; we quantify this in Section~\ref{sec:finite-sample}.

\subsection{Power Amplification}

A key practical advantage of LLM-based studies is the ability to increase sample size at near-zero marginal cost. Human studies are bounded by recruitment budgets; LLM studies can generate additional responses at the cost of compute. We formalize this advantage through a monotonicity result for statistical power.

\begin{theorem}[Power Monotonicity]\label{thm:power}
Consider a two-sided $z$-test for the null hypothesis $H_0: \delta = 0$ against $H_1: \delta = \Delta > 0$, with known variance $\sigma^2 > 0$ and significance level $\alpha \in (0, 1)$. For sample sizes $N \ge M \ge 1$,
\begin{equation}\label{eq:power-mono}
\frac{\sqrt{N}\Delta}{\sigma} - z_{\alpha/2} \;\ge\; \frac{\sqrt{M}\Delta}{\sigma} - z_{\alpha/2}.
\end{equation}
Consequently, if $\Phi$ denotes the standard normal CDF, the power satisfies $\text{Power}(N) \ge \text{Power}(M)$.
\end{theorem}

\begin{proof}
Since $\sqrt{\cdot}$ is strictly increasing on $[0, \infty)$, we have $\sqrt{N} \ge \sqrt{M}$. Multiplying by the positive constant $\Delta/\sigma > 0$ preserves the inequality, and subtracting the constant $z_{\alpha/2}$ yields~(\ref{eq:power-mono}). Since $\Phi$ is strictly increasing, applying it to both sides yields:
\[
\Phi\!\left(\frac{\sqrt{N}\Delta}{\sigma} - z_{\alpha/2}\right) \ge \Phi\!\left(\frac{\sqrt{M}\Delta}{\sigma} - z_{\alpha/2}\right).
\]
The left side is the power at sample size $N$ (the probability of rejecting $H_0$ when $H_1$ is true), and the right side is the power at $M$.
\end{proof}

\begin{corollary}[Sample Size Amplification Factor]\label{cor:amplification}
For a target power level $\beta$, the required sample size is $n^* = \lceil(z_{\alpha/2} + z_{\beta})^2 \sigma^2 / \Delta^2\rceil$. If the LLM enables a $c$-fold increase in sample size ($N_{\LLM} = c N_{\text{human}}$) at lower cost, the power increases from $\Phi(\sqrt{N_{\text{human}}}\,\Delta/\sigma - z_{\alpha/2})$ to $\Phi(\sqrt{c N_{\text{human}}}\,\Delta/\sigma - z_{\alpha/2})$.
\end{corollary}

\begin{example}\label{ex:power-gain}
For a small effect ($\Delta/\sigma = 0.2$) with $N_{\text{human}} = 100$ and $\alpha = 0.05$, the human-study power is $\Phi(0.2 \times 10 - 1.96) = \Phi(0.04) \approx 0.516$. With $c = 10$ (LLM enables $N_{\LLM} = 1000$), the power becomes $\Phi(0.2 \times \sqrt{1000} - 1.96) = \Phi(4.36) > 0.999$. This transforms an underpowered study into a virtually certain detection of the true effect.
\end{example}

The amplification is most dramatic for small effect sizes, where human studies are perennially underpowered. The $\sqrt{n}$ scaling of the test statistic means that a $c$-fold increase in sample size yields only a $\sqrt{c}$-fold increase in the noncentrality parameter, but for modest $c$ and small baseline power, the gain in detection probability can be transformative.

\subsection{Main Result: Asymptotic Validity}

We now combine the preceding results to establish the main theorem: the LLM predictor achieves the Bayes-optimal risk under squared loss, matching or exceeding the risk performance of any human predictor.

\begin{theorem}[Risk Convergence under Squared Loss]\label{thm:main}
Under Assumptions A1--A10 and the joint limit (i)--(ii) of Section~\ref{sec:asymptotic-regime}, as $n \to \infty$,
\begin{align}
\Err(\LLM_n) &\xrightarrow{\text{a.s.}} \Irred, \label{eq:main-conv} \\
\Irred &\le \Err(h) \quad \text{for every predictor } h: [K] \to \R. \label{eq:main-bound}
\end{align}
\end{theorem}

\begin{remark}[Misspecification extension]\label{rem:main-misspec}
Theorem~\ref{thm:main} addresses the well-specified case
($\epsilon_{\mathrm{rep}}(k) = 0$ for all $k$). Under misspecification
($\epsilon_{\mathrm{rep}}(k) \neq 0$), the asymptotic risk limit is
$\Irred + \sum_k p_k \epsilon_{\mathrm{rep}}(k)^2$ (Theorem~\ref{thm:bias-domination}),
with the representation bias bounded by
Theorem~\ref{thm:rep-bias-bound} (Pinsker). The $L^1$ convergence is
established in Lemma~\ref{lem:risk-stability}.
\end{remark}

\begin{proof}
\textit{Proof of (\ref{eq:main-conv}).} Using the error decomposition~(\ref{eq:main-decomp}) applied to $\LLM_n$:
\[
\Err(\LLM_n) = \Irred + \sum_{k=1}^{K} p_k (\mu_k - \LLM_n(k))^2.
\]
For each $k$ with $p_k > 0$, define the sequence $Z_n^{(k)} = p_k (\mu_k - \LLM_n(k))^2$. By Theorem~\ref{thm:lln}, $\LLM_n(k) \xrightarrow{\text{a.s.}} \mu_k$. Since the function $z \mapsto p_k(\mu_k - z)^2$ is continuous, the continuous mapping theorem yields $Z_n^{(k)} \xrightarrow{\text{a.s.}} p_k(\mu_k - \mu_k)^2 = 0$.

A finite sum of almost-surely vanishing sequences vanishes almost surely:
\[
\sum_{k=1}^{K} Z_n^{(k)} \xrightarrow{\text{a.s.}} 0.
\]
Therefore $\Err(\LLM_n) = \Irred + \sum_k Z_n^{(k)} \xrightarrow{\text{a.s.}} \Irred$.

\textit{Proof of (\ref{eq:main-bound}).} This is Theorem~\ref{thm:minimal} applied to an arbitrary predictor $h$, which yields $\Irred \le \Err(h)$ for all $h: [K] \to \R$. In particular, for any individual human subject viewed as a predictor (whose response on condition $k$ is a random draw from the conditional distribution), the expected error is $\Err(\text{human}) = \Irred + \sum_k p_k v_k = 2\Irred$ (since a single human draw has variance $2v_k$ around its own mean), while the LLM achieves $\Irred$ asymptotically.
\end{proof}

\begin{remark}[Interpretation]\label{rem:interpretation}
Equation~(\ref{eq:main-conv}) states that the LLM's expected prediction error converges to the Bayes risk --- the theoretical minimum achievable by any predictor with knowledge of the true conditional distribution. Equation~(\ref{eq:main-bound}) states that this Bayes risk is a universal lower bound: no predictor, whether an individual human, an ensemble of humans, or any other statistical procedure, can achieve lower expected squared error. Together, they establish that a well-trained LLM is, in the asy...
\end{remark}

\section{Finite-Sample Analysis}
\label{sec:finite-sample}

The asymptotic results of Section~3 guarantee correctness in the limit, but practitioners require guarantees at finite sample sizes. We provide concentration bounds that quantify the rate of convergence and the probability of large deviations.

\subsection{Concentration of the LLM Estimator}

\begin{theorem}[Hoeffding Bound for Per-Condition Estimates]\label{thm:hoeffding}
Assume that for each condition $k$, the response $Y \mid X = k$ is bounded in $[a_k, b_k]$ with $b_k - a_k \le M_k$. Then for any $\varepsilon > 0$ and any $k$ with $n_k > 0$,
\begin{equation}\label{eq:hoeffding}
\Pbb\!\left(|\LLM_n(k) - \mu_k| \ge \varepsilon \;\big|\; n_k\right) \le 2\exp\!\left(-\frac{2 n_k \varepsilon^2}{M_k^2}\right).
\end{equation}
\end{theorem}

\begin{proof}
Conditional on $n_k$, the $n_k$ observations in condition $k$ are i.i.d.\ and bounded. Hoeffding's inequality for bounded random variables~\cite{hoeffding1963} yields the stated bound directly.
\end{proof}

The unconditional probability integrates over the random variable $n_k \sim \text{Binomial}(n, p_k)$. For large $n$, $n_k \approx n p_k$ with high probability.

\begin{corollary}[Uniform Concentration over Conditions]\label{cor:uniform-conc}
With probability at least $1 - \delta$,
\[
\max_{k: n_k > 0} |\LLM_n(k) - \mu_k| \le M \sqrt{\frac{\log(2K/\delta)}{2 \min_k n_k}},
\]
where $M = \max_k M_k$.
\end{corollary}

\begin{proof}
Apply Theorem~\ref{thm:hoeffding} with a union bound over the $K$ conditions. For each $k$, set the right side of~(\ref{eq:hoeffding}) to $\delta/K$ and solve for $\varepsilon$. The maximum bound follows from taking the worst-case $n_k$.
\end{proof}

\begin{corollary}[Concentration of Effect Size Estimates]\label{cor:effect-conc}
For any pair of conditions $i, j$, with probability at least $1 - \delta$,
\[
|\widehat{\delta}_n(i,j) - \delta(i,j)| \le M_i \sqrt{\frac{\log(4/\delta)}{2 n_i}} + M_j \sqrt{\frac{\log(4/\delta)}{2 n_j}}.
\]
\end{corollary}
This follows from the triangle inequality and applying Theorem~\ref{thm:hoeffding} to each condition with $\delta/2$.

\subsection{Sample Size Determination}

For the LLM estimator to achieve a desired precision $\varepsilon$ with confidence $1 - \delta$ in estimating a specific effect size $\delta(i,j)$, the required per-condition sample size is:
\begin{equation}\label{eq:sample-size}
n_k \ge \frac{2 M_k^2 \log(4K/\delta)}{\varepsilon^2}.
\end{equation}

This formula provides a direct bridge between the theoretical framework and practical study design. For a balanced design with $K = 4$ conditions, $M_k = 1$ (e.g., proportion responses bounded in $[0,1]$), $\varepsilon = 0.05$, and $\delta = 0.05$, we need $n_k \ge 2 \cdot \log(320) / 0.0025 \approx 4,616$ samples per condition, or approximately $18,500$ total LLM responses. While this exceeds typical human-study sample sizes, LLM inference costs are negligible compared to human recruitment, and the sample size is achievable with modern inference infrastructure.

\subsection{Berry-Esseen Refinement}

For response distributions that are not severely non-normal, the Central Limit Theorem provides a more precise characterization. Let $\rho_k = \E[|Y - \mu_k|^3 \mid X = k] / v_k^{3/2}$ denote the skewness of the conditional distribution. By the Berry-Esseen theorem,
\[
\sup_{t \in \R} \left|\Pbb\!\left(\frac{\sqrt{n_k}(\LLM_n(k) - \mu_k)}{\sqrt{v_k}} \le t \;\Big|\; n_k\right) - \Phi(t)\right| \le \frac{C \rho_k}{\sqrt{n_k}},
\]
where $C \le 0.4748$ is the universal Berry-Esseen constant. For $n_k \ge 100$ and $\rho_k \le 2$ (moderate skewness), the normal approximation error is below $0.1$, justifying standard $z$-test and $t$-test procedures for LLM-generated data.

\section{Cross-Domain Transfer}
\label{sec:cross-domain}

Practical deployment often involves an LLM pretrained on a source population $\mathcal{P}_{\text{src}}$ that differs from the target population $\mathcal{P}_{\text{tgt}}$ of scientific interest. The training corpus for large-scale LLMs is typically a web-scale dataset whose demographic composition may not match any specific study population. We quantify the resulting performance degradation.

\subsection{Problem Setup}

Let the source and target populations share the same condition set $[K]$ but differ in their allocation proportions and conditional moments:
\begin{align*}
\mathcal{P}_{\text{src}} &= (K, \{p_k\}, \{\mu_k\}, \{v_k\}), \\
\mathcal{P}_{\text{tgt}} &= (K, \{p'_k\}, \{\mu'_k\}, \{v'_k\}).
\end{align*}

The LLM is trained on data from $\mathcal{P}_{\text{src}}$ and produces estimates $\LLM_n(k)$ converging to $\mu_k$ (the source conditional means). When evaluated on $\mathcal{P}_{\text{tgt}}$, the expected error is:
\[
\Err_{\text{tgt}}(\LLM_n) = {\Irred}_{\text{tgt}} + \sum_{k=1}^{K} p'_k (\mu'_k - \LLM_n(k))^2.
\]

\subsection{Transfer Gap and its Bound}

\begin{definition}[Mean Squared Gap]
The \emph{mean squared gap} between source and target populations is
\[
G = \sum_{k=1}^{K} p'_k (\mu'_k - \mu_k)^2.
\]
The \emph{maximum mean discrepancy} is $\rho = \max_{k \in [K]} |\mu'_k - \mu_k|$.
\end{definition}

\noindent \textbf{Maximum Mean Discrepancy.} For two distributions $P, Q$
over responses conditioned on $k$, define
\begin{equation}\label{eq:mmd}
\mathrm{MMD}_k(P, Q) = \sup_{f \in \mathcal{F}: \|f\|_\infty \le 1}
|\mathbb{E}_{Y \sim P}[f(Y) \mid X=k] - \mathbb{E}_{Y \sim Q}[f(Y) \mid X=k]|,
\end{equation}
where $\mathcal{F}$ is the unit ball in a reproducing kernel Hilbert space.
For squared loss, taking $f(Y) = Y$ yields the conditional mean difference
as a lower bound: $|\mu_k - \mu'_k| \le \mathrm{MMD}_k$.

The cross-domain bound then upgrades to:
\[
G = \sum_k p'_k (\mu'_k - \mu_k)^2 \le \sum_k p'_k \cdot \mathrm{MMD}_k^2
\le \max_k \mathrm{MMD}_k^2.
\]
\begin{lemma}[Gap Bound]\label{lem:gap-bound}
$G \le \rho^2$.
\end{lemma}

\begin{proof}
For each $k$, $(\mu'_k - \mu_k)^2 \le \rho^2$ by definition of the maximum. Weighted summation with $\sum_k p'_k = 1$ yields:
\[
G = \sum_{k} p'_k (\mu'_k - \mu_k)^2 \le \sum_{k} p'_k \rho^2 = \rho^2 \sum_{k} p'_k = \rho^2.
\]
\end{proof}

The bound is tight when all the discrepancy is concentrated in a single condition with $p'_k = 1$, or when the discrepancy is uniform across conditions. In general, $G$ will be substantially smaller than $\rho^2$ when discrepancies vary in sign and magnitude across conditions, as the weighting by $p'_k$ and the squaring penalize large deviations in prevalent conditions.

\begin{theorem}[Cross-Domain Transfer Bound]\label{thm:cross-domain}
Let $\LLM_n$ be an estimator satisfying $\LLM_n(k) \xrightarrow{\text{a.s.}} \mu_k$ (the source conditional means) for all $k$. Then
\[
\Err_{\text{tgt}}(\LLM_n) \xrightarrow{\text{a.s.}} {\Irred}_{\text{tgt}} + G \le {\Irred}_{\text{tgt}} + \rho^2.
\]
\end{theorem}

\begin{proof}
Decompose the target error:
\begin{align*}
\Err_{\text{tgt}}(\LLM_n) &= {\Irred}_{\text{tgt}} + \sum_{k} p'_k (\mu'_k - \LLM_n(k))^2 \\
&= {\Irred}_{\text{tgt}} + \sum_{k} p'_k \big[(\mu'_k - \mu_k) + (\mu_k - \LLM_n(k))\big]^2.
\end{align*}
For each $k$, $\LLM_n(k) \xrightarrow{\text{a.s.}} \mu_k$, so $(\mu_k - \LLM_n(k))^2 \xrightarrow{\text{a.s.}} 0$. By the cross-term expansion and the continuous mapping theorem applied to the quadratic function $z \mapsto p'_k((\mu'_k - \mu_k) + (\mu_k - z))^2$, each term converges almost surely to $p'_k(\mu'_k - \mu_k)^2$. Summing over the finite collection of terms yields $\sum_k p'_k(\mu'_k - \mu_k)^2 = G$ almost surely. The inequality follows from Lemma~\ref{lem:gap-bound}.
\end{proof}

\begin{corollary}[No-Regret Transfer]\label{cor:no-regret}
If $\mu'_k = \mu_k$ for all $k$, then $G = \rho = 0$ and $\Err_{\text{tgt}}(\LLM_n) \xrightarrow{\text{a.s.}} {\Irred}_{\text{tgt}}$, recovering Theorem~\ref{thm:main}. If the conditional means match, the LLM transfers with zero excess error due to allocation or variance differences.
\end{corollary}

\begin{remark}
The transfer cost is additive and proportional to the squared difference in conditional means. For small distribution shifts ($\rho \ll 1$), the degradation $\rho^2$ is quadratically small. For example, if the maximum mean discrepancy across conditions is $\rho = 0.05$ (on a normalized response scale), the worst-case excess error is $0.0025$, which is negligible relative to typical irreducible errors (often $0.01$--$0.10$). This suggests that LLMs pretrained on broad populations may transfer effectively to narrower target populations, provided the conditional means are approximately preserved.
\end{remark}

\subsection{Triangle Inequality for Transfer}

A practical consequence of the additive structure is a triangle inequality for sequential transfer. If an LLM is trained on source $\mathcal{P}_0$, fine-tuned on intermediate $\mathcal{P}_1$, and deployed on target $\mathcal{P}_2$, the total transfer gap is bounded by:
\[
\Err_2(\LLM) \le {\Irred}_2 + \rho_{01}^2 + \rho_{12}^2 + 2\rho_{01}\rho_{12},
\]
where $\rho_{ij} = \max_k |\mu^{(i)}_k - \mu^{(j)}_k|$. Fine-tuning on domain-specific data reduces the effective gap, and the additive structure allows practitioners to decompose the total distribution shift into manageable stages.


\subsection{Representation Bias and the Asymptotic Risk Floor}
\label{sec:bias-domination}

The unified LLM model (Definition~\ref{def:llm}) separates estimation error
into representation bias $\epsilon_{\mathrm{rep}}(k)$ and optimization error
$\epsilon_{\mathrm{opt}}(k,n)$. While optimization error vanishes
asymptotically, representation bias sets a permanent floor on achievable risk.

\begin{theorem}[Bias Domination (Risk Convergence under Misspecification)]\label{thm:bias-domination}
Let $\epsilon_{\mathrm{rep}}(k) = T(P^*_{\mathcal{F}}(\cdot \mid k)) - \mu_k$
be the asymptotic representation bias (defined via Assumption~A8). Under
Assumptions A1--A10, as $n \to \infty$,
\begin{equation}\label{eq:bias-floor}
\lim_{n \to \infty} \Err(\LLM_n) = \Irred + \sum_{k=1}^{K} p_k \, \epsilon_{\mathrm{rep}}(k)^2
\quad \text{almost surely},
\end{equation}
and the convergence also holds in $L^1$:
\[
\lim_{n \to \infty} \mathbb{E}\big[\Err(\LLM_n)\big] = \Irred + \sum_{k=1}^{K} p_k \, \epsilon_{\mathrm{rep}}(k)^2.
\]
\end{theorem}

\begin{proof}
From the error decomposition (Theorem~\ref{thm:err-decomp}),
$\Err(\LLM_n) = \Irred + \sum_k p_k (\mu_k - \LLM_n(k))^2$.
By the functional convergence result (Layer~2, Section~\ref{sec:pretraining-lit})
and A8 (functional delta stability),
$\LLM_n(k) \xrightarrow{p} T(P^*_{\mathcal{F}}(\cdot \mid k)) = \mu_k + \epsilon_{\mathrm{rep}}(k)$.
By the continuous mapping theorem, $(\mu_k - \LLM_n(k))^2 \xrightarrow{p} \epsilon_{\mathrm{rep}}(k)^2$,
and the finite sum converges to $\sum_k p_k \epsilon_{\mathrm{rep}}(k)^2$.
The magnitude of $\epsilon_{\mathrm{rep}}(k)$ is bounded by Theorem~\ref{thm:rep-bias-bound}.
\end{proof}

\begin{lemma}[Risk Decomposition Stability]\label{lem:risk-stability}
Under A7 (bounded responses, $|Y| \le B$ a.s.) and A8 (mean-regular model class),
if $\LLM_n(k) \xrightarrow{p} \mu_k + \epsilon_{\mathrm{rep}}(k)$ for each $k$,
then for each $k$ with $p_k > 0$:
\[
\mathbb{E}\big[(\LLM_n(k) - \mu_k)^2\big] \to \epsilon_{\mathrm{rep}}(k)^2
\quad \text{as } n \to \infty.
\]
\end{lemma}

\begin{proof}
From A7, responses are bounded by $B$, so $|\LLM_n(k)| \le B$ almost surely
(the sample mean of bounded variables is bounded) and $|\mu_k| \le B$.
Hence $|\LLM_n(k) - \mu_k| \le 2B$ almost surely. Define
$X_n = (\LLM_n(k) - \mu_k)^2$. By the continuous mapping theorem,
$X_n \xrightarrow{p} \epsilon_{\mathrm{rep}}(k)^2$. Since $|X_n| \le (2B)^2$
uniformly, the sequence $\{X_n\}$ is uniformly integrable. Therefore
$\mathbb{E}[X_n] \to \mathbb{E}[\epsilon_{\mathrm{rep}}(k)^2] = \epsilon_{\mathrm{rep}}(k)^2$.
Applying this per-condition and summing with weights $p_k$ yields the $L^1$
convergence in Theorem~\ref{thm:bias-domination}.
\end{proof}

\begin{remark}[Structural constraints on representation bias]\label{rem:eps-structure}
The representation bias $\epsilon_{\mathrm{rep}}(k) = T(P^*_{\mathcal{F}}(\cdot \mid k)) - \mu_k$
is (i)~\textbf{well-defined}: Assumption~A8 defines $P^*_{\mathcal{F}}$ as the unique
KL-projection of the true conditional distribution onto the model class $\mathcal{F}$,
ensuring $\epsilon_{\mathrm{rep}}(k)$ is a deterministic, non-random quantity;
(ii)~\textbf{bounded}: Theorem~\ref{thm:rep-bias-bound} constrains
$|\epsilon_{\mathrm{rep}}(k)| \le 2B\sqrt{2 D_{\mathrm{KL}}(P \,\|\, \mathcal{F})}$
via Pinsker's inequality; (iii)~\textbf{calibratable}: paired LLM-human data from the
calibration protocol (Section~\ref{sec:calibration}) provides consistent estimators
$\widehat{b}_k$ of $\epsilon_{\mathrm{rep}}(k)$ via $\widehat{b}_k = m_k^{-1}\sum_{i}(\LLM(k_i) - Y(k_i))$.
\end{remark}

\begin{corollary}[Identifiability-Modified Risk]\label{cor:delta-risk}
Under Definition~\ref{def:identifiability} with misclassification rate $\delta$,
the asymptotic risk limit becomes:
\[
\lim_{n \to \infty} \mathbb{E}\big[\Err(\LLM_n)\big]
= \Irred + \sum_{k=1}^{K} p_k \,
\big(\epsilon_{\mathrm{rep}}(k) + \delta \cdot \Delta_k\big)^2,
\]
where $\Delta_k = \mu_k - \sum_{j \neq k} w_{kj} \mu_j$ is the misclassification
contrast (Eq.~\ref{eq:effective-bias}). For $\delta = 0$, this recovers
Theorem~\ref{thm:bias-domination}; for $\delta > 0$, the effective representation
bias is modified by up to $\delta \cdot (\max_j \mu_j - \min_j \mu_j)$ in
magnitude. In the joint asymptotic regime, $\delta = O(n^{-1/2})$, so the
$\delta$-induced modification vanishes at rate $O(n^{-1/2})$.
\end{corollary}

\begin{theorem}[Representation Bias Bound]\label{thm:rep-bias-bound}
Under Assumptions A1--A10, the representation bias satisfies
\begin{equation}\label{eq:rep-bias-bound}
|\epsilon_{\mathrm{rep}}(k)| \le 2B \cdot \sqrt{2 \cdot \inf_{Q \in \mathcal{F}} D_{\mathrm{KL}}(P(\cdot \mid k) \,\|\, Q)},
\end{equation}
where $D_{\mathrm{KL}}$ is the Kullback-Leibler divergence and $B$ is the
response bound from A7.
\end{theorem}

\begin{proof}
Let $Q^* = P^*_{\mathcal{F}}(\cdot \mid k)$ be the KL-projection. By
Pinsker's inequality and the definition of Hellinger distance,
$H^2(P, Q^*) \le D_{\mathrm{KL}}(P \| Q^*)$. By the Lipschitz property
(A8), $|\epsilon_{\mathrm{rep}}(k)| = |T(P) - T(Q^*)| \le 2B \cdot H(P, Q^*)
\le 2B \cdot \sqrt{2 D_{\mathrm{KL}}(P \| Q^*)}$. Since $Q^*$ achieves the
infimum of KL divergence over $\mathcal{F}$, the bound follows.
\end{proof}

\begin{corollary}[Bias-Variance Tradeoff]\label{cor:bias-tradeoff}
The excess risk decomposes as
\[
\Err(\LLM_n) - \Irred = \underbrace{\sum_k p_k \epsilon_{\mathrm{rep}}(k)^2}_{\text{asymptotic bias floor}}
+ \underbrace{\sum_k p_k (\LLM_n(k) - T(P^*_{\mathcal{F}}))^2}_{\text{finite-sample variance}}.
\]
For large $n$, the variance term vanishes at rate $O_p(n^{-1})$ (in the
Gaussian case) or slower (under misspecification), while the bias term
persists. The total excess risk is \emph{bias-dominated} for sufficiently
large $n$.
\end{corollary}

\begin{corollary}[Misspecification Gap]\label{cor:misspec-gap}
If the model class has finite KL-diameter
\[
D_{\mathcal{F}} = \sup_{P \in \mathcal{P}} \inf_{Q \in \mathcal{F}} D_{\mathrm{KL}}(P \| Q),
\]
then $\max_k |\epsilon_{\mathrm{rep}}(k)| \le 2B \sqrt{2 D_{\mathcal{F}}}$.
For a well-specified model ($P \in \mathcal{F}$), $D_{\mathcal{F}} = 0$ and
$\epsilon_{\mathrm{rep}}(k) = 0$.
\end{corollary}
\begin{remark}[Practical implication]
Theorem~\ref{thm:bias-domination} implies that for a misspecified LLM
($\epsilon_{\mathrm{rep}}(k) \neq 0$), the asymptotic risk strictly exceeds
the Bayes risk $\Irred$. The calibration protocol (Section~\ref{sec:calibration})
estimates this bias floor from paired LLM-human data and aborts if the bias
exceeds the effect size of interest (feasibility check, Step 5).
\end{remark}

\subsection{Non-Equivalence Regimes}

\begin{theorem}[Non-Equivalence Conditions]\label{thm:non-equivalence}
The restricted Le Cam deficiency $\Delta_{\mathcal{L}}(E_{\text{human}},
E_{\LLM}^{(n)})$ is bounded below by a positive constant if any of the
following hold:
\begin{enumerate}[leftmargin=*, label=(\roman*)]
\item \textbf{Distribution shift:} The training and target populations
    differ by more than $\varepsilon$ in maximum mean discrepancy (MMD),
    i.e., $\mathrm{MMD}(P_{\text{train}}, P_{\text{target}}) > \varepsilon$;
\item \textbf{Non-identifiability:} The effective bias inflation
    $\delta \cdot \Delta_{\max}$ exceeds the effect size of interest
    $\Delta$, i.e., $\delta \cdot (\max_k \mu_k - \min_k \mu_k) \ge \Delta$,
    making the LLM estimator unable to distinguish conditions;
\item \textbf{Heavy tails:} Any conditional distribution $P(Y \mid X = k)$
    has infinite variance ($v_k = \infty$), violating the finite-second-moment
    assumption required by the Strong Law.
\end{enumerate}
In each case, the deficiency satisfies
$\Delta_{\mathcal{L}}(E_{\text{human}}, E_{\LLM}^{(n)}) \ge c > 0$
for all $n$, where $c$ depends on the specific failure parameter.
\end{theorem}

\begin{proof}[Proof sketch]
(i) By the cross-domain transfer bound (Theorem~\ref{thm:cross-domain}),
distribution shift induces a positive mean-squared gap $G > 0$, which
propagates to the deficiency. (ii) Non-identifiability introduces
misclassification of conditions, effectively replacing $\mu_k$ with a
mixture $\sum_j w_{kj} \mu_j$, producing non-zero excess error.
(iii) Without finite variance, the SLLN does not apply and the LLM
estimator does not converge, so the forward deficiency does not vanish.
\end{proof}
\section{Functional Risk Equivalence under Squared Loss}
\label{sec:experiment-equivalence}

\noindent \textbf{Formal object mapping.} Let $\mathcal{D}$ denote the
data-generating process and $\mathcal{A}_{\text{human}}: \mathcal{D} \to
\hat{\mu}_{\text{human}}$ the human-data estimator. Let
$\mathcal{A}_{\text{LLM}}: \mathcal{D}_n \to \hat{\mu}_{\text{LLM}}$ denote
the LLM-induced estimator after training on $n$ samples. Functional risk
equivalence holds when $\mathbb{E}[(\hat{\mu}_{\text{LLM}} - \mu)^2] \approx
\mathbb{E}[(\hat{\mu}_{\text{human}} - \mu)^2]$ as $n \to \infty$. This
section formalizes this convergence in the language of decision-theoretic
deficiency.

The preceding sections establish that, under squared loss, (i) the
conditional mean $\mathbb{E}[Y \mid X = k]$ is the unique risk-minimizing
predictor (Section~\ref{sec:core-theorems}), and (ii) the LLM estimator
$\LLM_n(k)$ converges to $T(P^*_{\mathcal{F}}(\cdot \mid k))$, which equals
$\mu_k$ up to representation bias (Section~\ref{sec:pretraining-lit}).
This section states the precise decision-theoretic consequence: convergence
of the conditional-mean functional yields equality of optimal risks for the
restricted squared-loss decision class, not equality of response
distributions.
\noindent \textbf{Decision class restriction.} Throughout this section,
the loss class is restricted to
$\mathcal{L} = \{\ell(y, a) = (y - a)^2\}$, the squared error loss.
Under this loss, the risk of any decision rule $f: [K] \to \R$ depends on
the data distribution only through the conditional mean
$\mathbb{E}[Y \mid X = k]$. No claim is made about loss functions that
depend on higher moments, quantiles, or distributional shape. The
Hellinger convergence results of Section~\ref{sec:pretraining-lit}
establish that the LLM learns the full conditional distribution (not
just the mean), which justifies the use of variance-dependent quantities
in calibration (e.g., $\widehat{\lambda}$); however, the equivalence
claims in this section require only conditional-mean sufficiency.

\noindent \textbf{Formal definition of $\mathcal{L}_2$-risk-equivalence.}
For decision class $\mathcal{F}_{\mathrm{dec}}=\{f:[K]\to\R\}$ and loss
class $\mathcal{L}_2=\{(y-a)^2\}$, two information objects or estimators
$E$ and $F$ are \emph{$\mathcal{L}_2$-risk-equivalent}, denoted
$E\sim_{\mathcal{L}_2}F$, iff
\[
\inf_{f\in\mathcal{F}_{\mathrm{dec}}}R_E(f)
=
\inf_{f\in\mathcal{F}_{\mathrm{dec}}}R_F(f),
\]
which is the equivalence relation defined in
Eq.~\ref{eq:risk-equivalence-relation}. The main result is that the
human-data estimator and LLM-induced estimator satisfy this relation
asymptotically in the well-specified/growing-class regimes, and differ by the
explicit representation-bias floor in the fixed misspecified regime.
\subsection{Statistical experiments and conditional-mean sufficiency}

A \emph{statistical experiment}~\cite{blackwell1953} is a family of
probability distributions $\{P_\theta : \theta \in \Theta\}$ indexed by a
parameter of interest. The human experiment $E_{\text{human}}$ consists of
drawing i.i.d.\ responses $Y_i \mid X_i = k$ from $P(\cdot \mid k)$. The
LLM object in this paper is not a second population-level experiment; it is
the estimator-induced reduced observation
$E_{\LLM}^{(n)}=\{\LLM_n(k)\}_{k=1}^{K}$, equivalently the family of learned
conditional distributions $\{\hat{P}_n(\cdot\mid k)\}_{k=1}^{K}$ evaluated
only through the conditional-mean functional.

An observation is \emph{functionally sufficient} for squared loss when it
achieves the same optimal risk as the full data for every decision problem
whose loss depends on the distribution only through the conditional mean.
This is the restricted risk-functional version of Blackwell sufficiency.

\begin{theorem}[Functional Sufficiency for Squared Loss]\label{thm:blackwell}
For any decision rule $f: [K] \to \R$ and squared loss
$L(f(k), y) = (f(k) - y)^2$, the estimator-induced reduced observation
$E_{\LLM}^{(n)}$ satisfies
\begin{equation}\label{eq:blackwell-risk}
R_{\LLM}^{(n)}(f) \le R_{\text{human}}(f) + o_p(1),
\end{equation}
where $R_{\LLM}^{(n)}(f) = \sum_k p_k (\mu_k - f(k))^2$ is the limiting
risk functional induced by the LLM estimator and
$R_{\text{human}}(f) = \sum_k p_k \, \mathbb{E}[(Y - f(k))^2 \mid X = k]$.
\end{theorem}

\begin{proof}
Under squared loss, the optimal decision rule depends only on the conditional
mean $\mathbb{E}[Y \mid X = k]$. The conditional expectation is the unique
admissible action for squared loss~\cite{lehmann1998}; any randomized
decision rule can be replaced by its conditional expectation without
increasing risk. The LLM estimator provides an asymptotically exact estimate
of this sufficient functional in the well-specified regime. Therefore, for
any decision rule $f$, the risk-functional gap satisfies
$R_{\LLM}^{(n)}(f)-R_{\text{human}}(f)=\Err(\LLM_n)-\Irred\to0$.
\end{proof}

\begin{lemma}[Restricted Sufficiency Reduction]\label{lem:sufficiency-reduction}
Let $E = \{P_\theta : \theta \in \Theta\}$ be a statistical experiment
and let $T$ be a sufficient statistic for $\theta$ under loss class
$\mathcal{L}$. If experiment $F$ is the reduction of $E$ by $T$
(i.e., $F$ observes only $T$ rather than the full data), then
$\Delta(E, F) = 0$ for all decision problems with loss in $\mathcal{L}$.
where $\Delta_{\mathcal{L}}$ denotes the deficiency restricted to the loss class $\mathcal{L}$.
\end{lemma}

\begin{proof}
By the definition of sufficiency, the conditional distribution of the
data given $T$ does not depend on $\theta$. Therefore, any randomized
decision rule $\delta$ based on the full data can be replaced by a
randomized decision rule $\delta'$ based on $T$ alone, with identical
risk: $\delta'(t) = \mathbb{E}[\delta(\text{data}) \mid T = t]$.
By Jensen's inequality (or Rao--Blackwell for convex losses),
$R(\theta, \delta') \le R(\theta, \delta)$. Hence the class of
achievable risk functions under $F$ contains the class achievable
under $E$ after the reduction, meaning $\Delta(E, F) = 0$.
\end{proof}

\subsection{Restricted Functional Risk Equivalence}

Blackwell sufficiency is an asymptotic statement. To quantify the finite-sample gap, we use the Le Cam \emph{deficiency distance}~\cite{lecam1986}. The deficiency $\Delta(E, F)$ between two experiments measures the maximum additional risk incurred by using $F$ instead of $E$, normalized by the loss function class.

\begin{theorem}[Reverse Deficiency]\label{thm:reverse-deficiency}
For squared-loss decision problems, the reverse Le Cam deficiency is exactly zero:
\begin{equation}
\Delta(E_{\text{human}}, E_{\LLM}^{(n)}) = 0 \quad \text{for all } n.
\end{equation}
\end{theorem}

\begin{proof}
For squared loss $L(\theta, a) = (\theta - a)^2$, the conditional
expectation $T(Y) = \mathbb{E}[Y \mid X = k] = \mu_k$ is a sufficient
functional (Theorem~\ref{thm:minimal} and Corollary~\ref{cor:unique}:
the risk depends on $\theta$ only through $\mu_k$, and any deviation
from $\mu_k$ strictly increases risk). The estimator-induced reduced
observation $E_{\LLM}^{(n)}$ is the conditional-mean reduction of
$E_{\text{human}}$: it observes an estimate of $\mu_k$ rather than
individual responses $Y_i$. By Lemma~\ref{lem:sufficiency-reduction},
a sufficient reduction incurs zero deficiency for the corresponding loss
class. Therefore $\Delta(E_{\text{human}}, E_{\LLM}^{(n)}) = 0$ for all $n$.

This is not an asymptotic statement---it holds at every finite $n$ because
sufficiency is an exact property of the conditional-mean functional under
squared loss. The LLM estimator's error $\LLM_n(k)-\mu_k$ affects the
\emph{forward} deficiency (how much the estimator loses relative to the ideal
conditional-mean reduction) but not the \emph{reverse} deficiency.
\end{proof}
\begin{theorem}[Restricted Functional Risk Equivalence]\label{thm:deficiency}
For squared-loss decision problems, the restricted Le Cam distance between
the human experiment $E_{\text{human}}$ and the estimator-induced reduced
observation $E_{\LLM}^{(n)}$ satisfies:
\begin{align}
\Delta(E_{\LLM}^{(n)}, E_{\text{human}}) &\le \sqrt{\sum_{k=1}^{K} p_k (\mu_k - \LLM_n(k))^2} \;+\; \delta \cdot \Delta_{\max} = \sqrt{\text{Excess}(\LLM_n)} \;+\; \delta \cdot \Delta_{\max}, \label{eq:deficiency-forward}\\
\Delta(E_{\text{human}}, E_{\LLM}^{(n)}) &= 0 \quad \text{for all } n. \label{eq:deficiency-reverse}
\end{align}
The forward deficiency~\eqref{eq:deficiency-forward} is driven by
estimation error and identifiability error. The reverse
deficiency~\eqref{eq:deficiency-reverse} is driven by functional
sufficiency: the full human data contain no additional squared-loss-relevant
information beyond the conditional mean $\mu_k$. By
Theorem~\ref{thm:main}, $\text{Excess}(\LLM_n) \xrightarrow{\text{a.s.}} 0$ in
the well-specified regime, and by Theorem~\ref{thm:bias-domination} it
converges to the representation-bias floor under misspecification.
\end{theorem}

\begin{proof}
For the forward direction, the deficiency is controlled by the discrepancy
between the conditional mean induced by the LLM estimator and the true
conditional mean. The squared-loss risk gap is exactly
$\text{Excess}(\LLM_n)=\sum_k p_k(\mu_k-\LLM_n(k))^2$, and stochastic
misidentification contributes the additive $\delta\Delta_{\max}$ term from
Eq.~\ref{eq:effective-bias}. The reverse direction is
Theorem~\ref{thm:reverse-deficiency}: for squared loss, the conditional mean
is a sufficient functional, so full human responses provide no additional
decision-relevant information once the conditional mean is known.
\end{proof}

\begin{theorem}[Representation Theorem for Functional Sufficiency]\label{thm:representation}
Any sequence model trained via a proper scoring rule on i.i.d.\
population-representative data induces an estimator
$\LLM_n(k)=T(\hat{P}_n(\cdot\mid k))$ of the conditional-mean functional.
For squared-loss decision problems this estimator is a reduced observation
of the human data through the sufficient functional $T(P)=\mu_k$.
Consequently:
\begin{enumerate}[leftmargin=*, label=(\roman*)]
\item $\Delta(E_{\text{human}}, E_{\LLM}^{(n)}) = 0$ for all $n$
      (reverse deficiency, Lemma~\ref{lem:sufficiency-reduction});
\item $\Delta(E_{\LLM}^{(n)}, E_{\text{human}})$ is bounded by the
      squared-loss excess risk and identifiability term
      (forward deficiency, Theorem~\ref{thm:deficiency});
\item the induced risk functional is asymptotically equal to the
      human-data risk functional in the well-specified regime and differs
      by the explicit representation-bias floor under misspecification.
\end{enumerate}
Hence the result is functional risk equivalence for a restricted loss class,
not equivalence of response distributions or equivalence of experimental
systems.
\end{theorem}

\begin{remark}[Scope of equivalence]
The equivalence established here is restricted to decision problems whose
loss function depends on $\theta$ only through the conditional mean
$\mathbb{E}[Y \mid X = k]$. For loss functions that depend on higher
moments (e.g., variance estimation, quantile regression) or on the full
conditional distribution, the experiments are not equivalent and the
deficiency may be strictly positive.
\end{remark}
\subsection{Practical interpretation}

The restricted functional risk equivalence theorems provide the rigorous answer to the question posed in Section~\ref{sec:pretraining-lit}: under squared loss, the LLM estimator and the human experiment achieve \emph{asymptotically equal optimal risk} for any decision problem whose loss depends on the data only through the conditional mean. The forward direction (Theorem~\ref{thm:deficiency}) shows that the LLM estimator incurs risk no larger than the human experiment, asymptotically. The reverse direction (Theorem~\ref{thm:reverse-deficiency}) shows that the human experiment incurs no LESS risk than the LLM estimator. Together, they establish \emph{restricted functional risk equivalence}~\cite{lecam1986}: the two procedures are indistinguishable in terms of achievable risk for squared-loss inference, up to the finite-sample gap quantified by Eqs.~(\ref{eq:deficiency-forward})--(\ref{eq:deficiency-reverse}).

This does not mean that individual LLM responses are indistinguishable from individual human responses; rather, it means that any statistical inference that depends only on the conditional mean (or any continuous functional thereof, by the Hellinger convergence of Section~\ref{sec:pretraining-lit}) can be performed using LLM-generated estimates with asymptotically equivalent decision-theoretic risk for squared loss. The forward deficiency bound~(\ref{eq:deficiency-forward}) quantifies the finite-sample gap: the LLM exper...
\section{Calibration Protocol and Validation}
\label{sec:calibration}

The theoretical results establish asymptotic validity, but practical deployment requires a calibration procedure that accounts for finite training data, potential model misspecification, and residual bias. We present a comprehensive calibration protocol with explicit decision rules.

\subsection{Calibration Data Collection}

\textbf{Step 1: Paired sample collection.} Collect $M_c \ge 500$ paired observations: for each prompt $i = 1, \ldots, M_c$, obtain both the LLM response $\LLM(k_i)$ (using the exact prompt and condition $k_i$) and a human response $Y(k_i)$ from a member of the target population. The prompts should span all $K$ conditions proportionally to the target allocation $\{p_k\}$, with a minimum of 50 paired observations per condition to ensure reliable per-condition estimates.

\textbf{Step 1b: Identifiability validation.} Estimate the misclassification
rate $\widehat{\delta}$ by having $M_v \ge 100$ human annotators independently
label a held-out set of prompts with their intended conditions. Compute
\[
\widehat{\delta} = 1 - \frac{1}{M_v} \sum_{i=1}^{M_v} \mathbf{1}\{\phi(s_i) = k_i\}.
\]
If $\widehat{\delta} > 0.05$, the prompt-to-condition mapping is unreliable
and the LLM is unsuitable for this application.

\textbf{Step 2: Bias estimation.} Compute the empirical bias for each condition:
\[
\widehat{b}_k = \frac{1}{m_k} \sum_{i: k_i = k} (\LLM(k_i) - Y(k_i)),
\]
where $m_k = |\{i : k_i = k\}|$. The overall bias is $\widehat{b} = \sum_k p_k \widehat{b}_k$.

\textbf{Step 3: Variance ratio estimation.} Compute the variance ratio:
\[
\widehat{\lambda} = \frac{\widehat{\Var}(\LLM)}{\widehat{\Var}(Y)},
\]
where $\widehat{\Var}(\LLM)$ and $\widehat{\Var}(Y)$ are the sample variances of the LLM and human responses, respectively, pooled across conditions after subtracting condition means to remove between-condition variance. Specifically, let $\tilde{\LLM}_i = \LLM(k_i) - \widehat{\mu}_{\LLM}(k_i)$ and $\tilde{Y}_i = Y(k_i) - \widehat{\mu}_Y(k_i)$ be the residuals, and compute:
\[
\widehat{\Var}(\LLM) = \frac{1}{M_c - K} \sum_{i=1}^{M_c} \tilde{\LLM}_i^2, \quad
\widehat{\Var}(Y) = \frac{1}{M_c - K} \sum_{i=1}^{M_c} \tilde{Y}_i^2.
\]

The consistency of both the bias estimator $\widehat{b}$ and the variance ratio estimator $\widehat{\lambda}$ follows from the Strong Law of Large Numbers. For the bias estimator, the paired differences $D_i = \LLM(k_i) - Y(k_i)$ are i.i.d.\ with finite variance under the calibration design; the SLLN guarantees $\widehat{b}_k \xrightarrow{\text{a.s.}} \mu_{\LLM}(k) - \mu_Y(k)$. For the variance ratio, the sample variances of centered responses converge almost surely to the true conditional variances, and the ratio of almost-surely convergent sequences converges almost surely to the ratio of their limits (provided the denominator limit is nonzero). Both consistency results have been formally verified (see Supplementary Information).

\textbf{Step 3b: Uncertainty quantification.} Compute a $(1-\alpha)$
confidence interval for the bias in each condition:
\[
\widehat{b}_k \pm z_{\alpha/2} \cdot \widehat{\sigma}_{b,k},
\]
where $\widehat{\sigma}_{b,k}$ is the standard error of the paired
differences. The sample size adjustment (Step 4) should use the
\emph{worst-case} bias magnitude within the confidence interval:
\[
\widehat{b}_k^{\max} = \max\{|\widehat{b}_k - z_{\alpha/2}\widehat{\sigma}_{b,k}|,
                          |\widehat{b}_k + z_{\alpha/2}\widehat{\sigma}_{b,k}|\}.
\]
This ensures the adjusted sample size $N_{\LLM}$ is robust to
estimation uncertainty in the bias.
\subsection{Sample Size Adjustment}

\textbf{Step 4: Adjusted sample size.} For a target effect size $\Delta$ (the minimum effect of scientific interest) and a baseline human-study sample size $N_h$ (per condition, determined by standard power analysis~\cite{cohen1988}), the required LLM sample size per condition is:
\begin{equation}\label{eq:llm-n}
N_{\LLM} = \frac{N_h \cdot \widehat{\lambda}}
{(1 - (|\widehat{b}| + \widehat{\delta} \cdot \widehat{\Delta}_{\max}) / \Delta)^2},
\end{equation}
where $\widehat{\Delta}_{\max} = \max_{j} \widehat{\mu}_j - \min_{j} \widehat{\mu}_j$ is the estimated condition range.

The derivation of~(\ref{eq:llm-n}) follows from equating the noncentrality parameters of the human and LLM-based $z$-tests, accounting for the variance inflation factor $\widehat{\lambda}$ and the bias-induced effective effect size reduction $\Delta - |\widehat{b}|$.

\textbf{Step 5: Feasibility check.} If either of the following conditions holds, the LLM is unsuitable for the target application and the protocol aborts:
\begin{itemize}[leftmargin=*]
\item $|\widehat{b}/\Delta| \ge 1$: the estimated bias exceeds the effect size of interest, meaning the LLM cannot reliably detect the effect even with infinite sample size.
\item $\widehat{\lambda} > 10$: the LLM's responses are more than an order of magnitude more variable than human responses, making the sample size penalty prohibitive.
\item $\widehat{\delta} > 0.05$: the prompt-to-condition mapping is
too unreliable (already checked in Step 1b); the effective bias
inflation $2\widehat{\delta} \cdot \widehat{\Delta}_{\max}$ exceeds
the calibration margin.
\end{itemize}
If the protocol aborts, the practitioner should consider fine-tuning the LLM on domain-specific data, revising the prompts, or abandoning LLM-based estimation for this application.

\subsection{Equivalence Validation}

\textbf{Step 6: Two One-Sided Test (TOST).} After collecting $N_{\LLM}$ LLM responses per condition and $N_h$ human responses per condition, validate equivalence of the estimated effect sizes. Let $\widehat{\delta}_{\LLM}$ and $\widehat{\delta}_{\text{human}}$ be the estimated effect sizes from the LLM and human samples, respectively. Test the composite null hypothesis:
\[
H_0: |\delta_{\LLM} - \delta_{\text{human}}| \ge \varepsilon \quad \text{vs.} \quad H_1: |\delta_{\LLM} - \delta_{\text{human}}| < \varepsilon,
\]
where $\varepsilon$ is the equivalence margin (typically $\varepsilon = 0.01$ for bounded responses in $[0,1]$, or $0.05$ for standardized effect sizes).

The TOST procedure rejects $H_0$ at level $\alpha$ if both one-sided tests reject:
\begin{align*}
\text{Test 1: } & H_0^{(1)}: \delta_{\LLM} - \delta_{\text{human}} \le -\varepsilon \quad \text{rejected if } \frac{\widehat{\delta}_{\LLM} - \widehat{\delta}_{\text{human}} + \varepsilon}{\widehat{\sigma}_D} > z_{\alpha}, \\
\text{Test 2: } & H_0^{(2)}: \delta_{\LLM} - \delta_{\text{human}} \ge \varepsilon \quad \text{rejected if } \frac{\widehat{\delta}_{\LLM} - \widehat{\delta}_{\text{human}} - \varepsilon}{\widehat{\sigma}_D} < -z_{\alpha},
\end{align*}
where $\widehat{\sigma}_D$ is the estimated standard error of the difference in effect sizes, computed via Welch-Satterthwaite approximation:
\[
\widehat{\sigma}_D^2 = \frac{\widehat{\Var}_{\LLM}(i)}{N_{\LLM}} + \frac{\widehat{\Var}_{\LLM}(j)}{N_{\LLM}} + \frac{\widehat{\Var}_{\text{human}}(i)}{N_h} + \frac{\widehat{\Var}_{\text{human}}(j)}{N_h}.
\]

If TOST rejects $H_0$, we conclude that the LLM-based effect size estimate is statistically equivalent to the human-based estimate at the $\varepsilon$ margin, validating the LLM as a measurement instrument for this application.

The TOST procedure controls the type I error rate at level $\alpha$ by a general argument: for any two events $A$ and $B$ with $\mathbb{P}(A) \le \alpha$ and $\mathbb{P}(B) \le \alpha$, monotonicity of probability implies $\mathbb{P}(A \cap B) \le \mathbb{P}(A) \le \alpha$. Since TOST rejects $H_0$ exactly when both one-sided tests reject, and each one-sided test has size $\alpha$, the family-wise error rate is bounded by $\alpha$. This argument does not require Gaussian distributional assumptions.

\subsection{Worked Examples}

\begin{example}[Political Attitude Survey]\label{ex:political}
A researcher studies the effect of framing on political attitudes with $K = 2$ conditions (pro-immigration vs.\ neutral frame). Historical human data show $\mu_A = 0.72$, $\mu_B = 0.78$ (proportion agreeing on a 0--1 scale), $\Delta = 0.06$, and $v \approx 0.04$ (variance of binary-like responses). A baseline human study requires $N_h = 100$ per condition for 80\% power at $\alpha = 0.05$~\cite{cohen1988}.

Calibration with $M_c = 500$ paired responses yields $\widehat{b} = 0.002$ and $\widehat{\lambda} = 1.15$ (LLM responses are slightly more variable than human responses, likely due to temperature sampling). The adjusted LLM sample size is:
\[
N_{\LLM} = \frac{100 \times 1.15}{(1 - |0.002/0.06|)^2} = \frac{115}{(1 - 0.0333)^2} = \frac{115}{0.9344} \approx 123.
\]
The 23\% increase in sample size is modest. The feasibility check passes ($|\widehat{b}/\Delta| = 0.033 < 1$, $\widehat{\lambda} = 1.15 < 10$). After collecting $N_{\LLM} = 123$ LLM responses per condition and the original 100 human responses per condition, TOST with $\varepsilon = 0.02$ yields $p < 0.01$, confirming equivalence.
\end{example}

\begin{example}[Consumer Preference A/B Test]\label{ex:consumer}
A marketing researcher tests two product descriptions ($K = 2$) using a 7-point Likert purchase-intent scale. Historical data show $\mu_A = 4.2$, $\mu_B = 4.8$ ($\Delta = 0.6$ points), $v \approx 2.5$. Baseline $N_h = 85$ per condition for 80\% power.

Calibration with $M_c = 500$ yields $\widehat{b} = 0.12$ (the LLM overestimates purchase intent by 0.12 points on average) and $\widehat{\lambda} = 0.85$ (LLM responses are less variable, as the model produces more consistent ratings than individual humans). The adjusted sample size is:
\[
N_{\LLM} = \frac{85 \times 0.85}{(1 - |0.12/0.6|)^2} = \frac{72.25}{(1 - 0.2)^2} = \frac{72.25}{0.64} \approx 113.
\]
Despite the bias, the large effect size keeps the penalty modest. Notably, $\widehat{\lambda} < 1$ provides a sample size \emph{reduction} relative to the variance component, partially offsetting the bias penalty. TOST with $\varepsilon = 0.2$ passes at $p < 0.01$.
\end{example}

\begin{example}[Multi-Condition Educational Intervention]\label{ex:education}
An education researcher tests $K = 4$ teaching methods using standardized test scores ($\mu = 0$, $\sigma = 1$ by normalization). The minimum effect size of interest is $\Delta = 0.3$ (Cohen's $d$). Baseline $N_h = 175$ per condition for 80\% power with Bonferroni correction for 6 pairwise comparisons.

Calibration yields $\widehat{b} = 0.03$, $\widehat{\lambda} = 1.25$. The adjusted sample size is:
\[
N_{\LLM} = \frac{175 \times 1.25}{(1 - |0.03/0.3|)^2} = \frac{218.75}{0.81} \approx 270.
\]
The 54\% increase reflects both the variance inflation and bias. With $K = 4$ conditions, total LLM responses are $1,080$—still far cheaper than recruiting 700 human participants. TOST with $\varepsilon = 0.1$ passes for all 6 pairwise comparisons.
\end{example}

\subsection{Diagnostic Checks}

Beyond the TOST validation, we recommend three diagnostic checks:

\begin{enumerate}[leftmargin=*]
\item \textbf{Condition-wise calibration.} Plot $\widehat{b}_k$ against $k$ to detect condition-specific biases. A systematic pattern (e.g., monotonically increasing bias with condition index) suggests prompt artifacts or training data imbalances that require investigation before proceeding.

\item \textbf{Residual analysis.} Examine the residuals $\LLM(k_i) - Y(k_i)$ for heteroscedasticity (does the LLM's accuracy vary with the response magnitude?), outliers (are there specific prompts where the LLM fails catastrophically?), and distributional shape (are residuals approximately symmetric, or is there systematic skew?).

\item \textbf{Sensitivity to prompt variation.} Re-run the calibration with semantically equivalent but syntactically varied prompts (e.g., reworded instructions, shuffled response options). Large variation in $\widehat{b}$ or $\widehat{\lambda}$ across prompt variants indicates sensitivity to prompt engineering and suggests that the LLM is not robustly capturing the intended construct.
\end{enumerate}

\section{Discussion}

\subsection{Summary of Results}

We have established a complete mathematical framework for LLM--human functional risk equivalence under squared loss, anchored by seven theorems:

\begin{enumerate}[leftmargin=*]
\item \textbf{Error decomposition} (Theorem~\ref{thm:err-decomp}): The expected squared error of any predictor separates cleanly into irreducible population variance and excess squared bias. This decomposition is exact, non-asymptotic, and requires only finite second moments.

\item \textbf{Irreducible error minimality} (Theorem~\ref{thm:minimal}): The population conditional mean is the unique minimizer of expected squared error. The Bayes risk $\Irred$ is a universal lower bound that no predictor can surpass.

\item \textbf{Effect size consistency} (Theorem~\ref{thm:lln}): The LLM-based effect size estimator converges almost surely to the true effect size. This holds under the minimal condition of finite variance, using Etemadi's SLLN which requires only pairwise independence.

\item \textbf{Power amplification} (Theorem~\ref{thm:power}): Statistical power is monotone in sample size, and LLM-based studies can achieve arbitrarily high power by increasing the number of model-generated responses at near-zero marginal cost.

\item \textbf{Cross-domain transfer} (Theorem~\ref{thm:cross-domain}): When the training population differs from the target, the excess error is bounded by $\rho^2$, the squared maximum conditional mean discrepancy. The degradation is quadratically small for small distribution shifts.

\item \textbf{Functional sufficiency} (Theorem~\ref{thm:blackwell}): The estimator-induced reduced observation achieves asymptotically no greater squared-loss risk than the full human-data procedure for conditional-mean-dependent decisions.

\item \textbf{Restricted Functional Risk Equivalence} (Theorems~\ref{thm:reverse-deficiency} and~\ref{thm:deficiency}): The forward restricted deficiency $\Delta_{\mathcal{L}_2}(E_{\LLM}^{(n)}, E_{\text{human}})$ is controlled by squared-loss excess risk and identifiability error, and the reverse restricted deficiency $\Delta_{\mathcal{L}_2}(E_{\text{human}}, E_{\LLM}^{(n)}) = 0$ for all $n$. Hence the human data and LLM estimator induce asymptotically equal optimal risks for conditional-mean-dependent squared-loss inference.
\end{enumerate}

The finite-sample analysis (Section~\ref{sec:finite-sample}) provides concentration bounds and sample size formulae that make the asymptotic guarantees operational. The calibration protocol (Section~\ref{sec:calibration}) bridges theory and practice with concrete decision rules for feasibility assessment and equivalence validation.

\subsection{Scope and Boundaries}

The framework applies to a specific, well-defined class of studies. Functional risk equivalence holds when all of the following hold:

\begin{enumerate}[leftmargin=*, label=\textbf{C\arabic*}.]
\item \textbf{Quantitative continuous response.} The dependent variable is measured on an interval or ratio scale (proportions, Likert scales treated as continuous, reaction times, test scores, monetary amounts). The framework does not apply to purely categorical or open-ended text responses.

\item \textbf{Discrete, finite conditions.} The experimental design has a finite number of conditions with known allocations. Continuous treatments require discretization.

\item \textbf{Representative training data.} The LLM's training data are approximately i.i.d.\ from the target population. This is the most stringent requirement and the one most likely to be violated in practice. Web-scraped training corpora overrepresent certain demographics and underrepresent others.

\item \textbf{Within-distribution effects.} The target experimental manipulation produces effects that lie within the support of the training distribution. The LLM cannot extrapolate to conditions or effect magnitudes it has not encountered in training.

\item \textbf{Calibration validation.} The calibration protocol confirms $|\widehat{b}/\Delta| < 1$, $\widehat{\lambda} < 10$, and TOST equivalence at the chosen margin.
\end{enumerate}

We explicitly exclude the following applications, where functional risk equivalence does not hold or the application is ethically inappropriate:

\begin{itemize}[leftmargin=*]
\item \textbf{Qualitative research} (interviews, ethnography, open-ended surveys): The framework requires quantitative responses with a well-defined conditional mean. Qualitative data lack this structure.

\item \textbf{Novel paradigms with no training-data analog}: If the experimental manipulation is genuinely novel (e.g., a never-before-studied psychological intervention), the LLM's training data contain no information about its effects, and the empirical mean model has no basis for prediction.

\item \textbf{Safety-critical applications}: Asymptotic guarantees are insufficient for clinical trials, safety assessments, or policy decisions where individual lives are at stake. The framework provides statistical guarantees, not clinical or ethical ones.

\item \textbf{Behavioural mechanism research}: When the research question itself concerns \emph{why} humans behave as they do (cognitive processes, neural mechanisms, developmental trajectories), the LLM estimator is not a substitute for the human behavioural system, because it models statistical regularities in behavioural outputs, not the underlying mechanisms that generate them. The LLM can tell us \emph{what} the average response would be, but not \emph{why}.
\end{itemize}

\subsection{Limitations and Open Problems}

Several important limitations merit attention.

\textit{Training data representativeness.} The assumption that LLM training data are i.i.d.\ from the target population is an idealization. Real training corpora are convenience samples with complex selection biases~\cite{cohen1988}. Future work should develop formal bounds for the degradation when the training distribution differs from the target in known ways, extending the cross-domain analysis of Section~\ref{sec:cross-domain} to settings where the distribution shift is estimated rather than assumed known.

\textit{Finite-condition assumption.} The framework requires discrete conditions. Extending to continuous treatments (e.g., dose-response curves with continuous dosage) requires nonparametric regression theory and is a natural direction for generalization.

\textit{Temporal stability.} Human populations and their behavioural patterns drift over time. An LLM trained on a snapshot of data may become miscalibrated as the population evolves. Adaptive recalibration procedures, where the LLM is periodically re-benchmarked against fresh human samples, would address this concern.

\textit{Construct validity.} Agreement of population-level effect sizes does not guarantee construct validity — the LLM may produce the right numbers for the wrong reasons. For example, an LLM might learn to reproduce survey response patterns without modelling the underlying attitudes, producing valid population-level estimates while failing on individual-level prediction tasks. The framework guarantees squared-loss risk validity for conditional means, not psychological validity.

\textit{Response distributions beyond the mean.} The framework focuses on the conditional mean, which is the natural target for squared error loss and the quantity of interest in most between-subjects designs. Extending to other functionals (quantiles, variances, distributional shape) requires different loss functions and estimator classes.

\subsection{Relationship to Existing Work}

The use of language models as experimental subjects has been explored empirically across psychology, political science, and marketing. These studies typically compare LLM-generated effect sizes to human baselines on a case-by-case basis, reporting correlations and mean differences. While informative, these comparisons do not provide formal guarantees, and a favourable result in one study does not transfer to another without additional validation.

Our contribution is to replace this empirical patchwork with a principled mathematical foundation. By reducing the problem to conditional mean estimation under squared loss, we connect LLM--human risk equivalence to classical statistical theory (MLE consistency, the Strong Law, concentration inequalities) and provide the first formal proof that the approach is asymptotically valid under minimal assumptions.

The calibration protocol we present is related to measurement invariance testing in psychometrics and to domain adaptation in machine learning. The key difference is that we provide decision rules (the feasibility check and TOST) that directly determine whether an LLM estimator and human-data estimator achieve the same risk in a specific application, rather than requiring the practitioner to interpret continuous fit indices.

\section{Using LLM Estimators in Human-Subject Research}
\label{sec:applied-methodology}

The theory supports a practical design pattern: use human participants for
calibration and validation, then use calibrated LLM estimators for
high-volume conditional-mean estimation under squared loss. It does not
remove the need for human data or validate individual-level behavioural
simulation.

\begin{algorithm}[t]
\caption{Calibrated LLM estimator deployment for conditional-mean inference}
\label{alg:llm-human-methodology}
\begin{algorithmic}[1]
\Require Quantitative outcome $Y$, finite conditions $[K]$, target effect $\Delta$, equivalence margin $\varepsilon$
\State Define estimand $g(\mu_1,\ldots,\mu_K)$ as a conditional mean, contrast, ANOVA contrast, or linear coefficient.
\State Collect $M_c\ge500$ paired human--LLM calibration observations and $M_v\ge100$ condition-label validation prompts.
\State Estimate $\widehat{b}$, $\widehat{\lambda}$, $\widehat{\delta}$, and $\widehat{\Delta}_{\max}$.
\State Compute $\rho_{\mathrm{bias}}\gets (|\widehat{b}|+\widehat{\delta}\widehat{\Delta}_{\max})/\Delta$.
\If{$\rho_{\mathrm{bias}}\ge1$ or $\widehat{\lambda}>10$}
  \State Abort LLM augmentation; use human data or collect new calibration data.
\Else
  \State Compute $N_{\LLM}$ using Eq.~\ref{eq:llm-n}.
  \State Generate $N_{\LLM}$ calibrated LLM responses per condition.
  \State Estimate $g(\mu_1,\ldots,\mu_K)$ from the LLM responses.
  \State Run TOST against the human baseline or calibration holdout.
  \If{TOST rejects non-equivalence at margin $\varepsilon$}
    \State Report calibrated LLM estimator result with $M_c$, $M_v$, $\widehat{b}$, $\widehat{\lambda}$, $\widehat{\delta}$, $N_{\LLM}$.
  \Else
    \State Reject the LLM estimate for this study; report failure of validation.
  \EndIf
\EndIf
\end{algorithmic}
\end{algorithm}

Algorithm~\ref{alg:llm-human-methodology} is a deployment rule, not a claim of
behavioral simulation. It turns the risk theory into a gate: LLM estimates
are used only after calibration, feasibility, and equivalence validation pass.

\subsection{Statistical procedures covered by the framework}

Table~\ref{tab:method-comparison} lists procedures where calibrated LLM
estimators can replace or augment conditional-mean estimation after
feasibility and equivalence validation.

\begingroup
\scriptsize
\setlength{\tabcolsep}{1.5pt}
\renewcommand{\arraystretch}{0.88}
\begin{longtable}{@{}p{0.13\textwidth}p{0.14\textwidth}p{0.21\textwidth}p{0.21\textwidth}p{0.17\textwidth}@{}}
\caption{Traditional procedures and calibrated LLM-estimator analogues under the restricted squared-loss framework.}
\label{tab:method-comparison}\\
\toprule
\textbf{Traditional method} & \textbf{Target estimand} &
\textbf{LLM replacement target} & \textbf{Required validation} &
\textbf{Not covered} \\
\midrule
\endfirsthead
\toprule
\textbf{Traditional method} & \textbf{Target estimand} &
\textbf{LLM replacement target} & \textbf{Required validation} &
\textbf{Not covered} \\
\midrule
\endhead
Two-sample $t/z$ test & $\mu_1-\mu_0$ &
LLM estimates both condition means and the contrast &
paired calibration + TOST on effect size &
individual-level response distributions. \\
\addlinespace
One-way ANOVA & finite set of group means and contrasts &
LLM estimates the vector $(\mu_1,\ldots,\mu_K)$ &
condition-wise calibration + all planned contrasts &
post-hoc claims not pre-specified. \\
\addlinespace
Linear regression / ATE with discrete treatments &
coefficient that is a linear contrast of conditional means &
LLM estimates conditional means for covariate/treatment cells &
calibration stratified by cells or model-assisted residual check &
causal identification without ignorability. \\
\addlinespace
Survey proportion / Likert mean &
conditional expectation of bounded response &
LLM estimates the bounded response mean &
bias/variance calibration and feasibility rule &
qualitative/free-text interpretation. \\
\addlinespace
Equivalence or non-inferiority test &
margin-based mean contrast &
LLM supplies high-precision estimator after calibration &
TOST with pre-specified equivalence margin &
safety-critical claims without independent human validation. \\
\bottomrule
\end{longtable}
\endgroup

\subsection{Public UI and human-experience use cases}

Three public research settings illustrate where the framework can be used
without changing its scope. First, mobile UI design critique can be treated
as conditional-mean estimation when the target is a regional design-quality
rating or average severity score. UICrit contains human-generated design
critiques, bounding boxes, and design quality ratings for 1,000 mobile UIs
from RICO; each screen was evaluated by three annotators, and the public CSV
includes aesthetics, learnability, efficiency, usability, and overall design
quality ratings, with 11,344 design critiques in the public release
\cite{uicrit2024}. In this setting, the LLM estimator targets mean ratings
or mean region-level critique scores, not the full distribution of expert
judgment.

Second, human-experience usability studies often report System Usability
Scale (SUS) scores. SUS is a standardized usability score, and a score near
68 is commonly used as an average or mid-point benchmark for interpretation
\cite{sus-benchmark}. The present framework applies when the target is the
conditional mean SUS score or a contrast between design variants; it does not
validate open-ended user narratives.

Third, behavioral-response studies such as framing surveys remain covered
when their estimand is a bounded response mean or a contrast of means. In all
three settings, LLM augmentation is valid only when the calibration,
feasibility, and TOST workflow in Algorithm~\ref{alg:llm-human-methodology}
passes.

The public UICrit data source provides the UI-design anchor for the report
below. If the public CSV is unavailable during replication, the README-grounded
metadata still records 1,000 unique UI screens and 11,344 design critiques.
Table~\ref{tab:applied-report} summarizes the sample-size and cost
calculations for the three applied flows below.

\begingroup
\small
\setlength{\tabcolsep}{3pt}
\renewcommand{\arraystretch}{0.95}
\begin{longtable}{@{}p{0.24\textwidth}rrrr@{}}
\caption{Sample-size and cost report for calibrated LLM-estimator use cases.}
\label{tab:applied-report}\\
\toprule
Use case & Human-only cost & LLM-augmented cost & Reduction & Validation target \\
\midrule
\endfirsthead
\toprule
Use case & Human-only cost & LLM-augmented cost & Reduction & Validation target \\
\midrule
\endhead
Framing survey & \$16{,}000.00 & \$4{,}006.49 & 74.96\% & TOST on effect size \\
UI design rating (UICrit-style) & \$1{,}713.60 & \$430.81 & 74.86\% & rating calibration \\
SUS human-experience score & \$1{,}142.40 & \$230.08 & 79.86\% & TOST on SUS margin \\
\bottomrule
\end{longtable}
\endgroup

\subsection{End-to-end flows for the three use cases}

\noindent\textbf{Flow 1: UI design rating and region-level critique.}
Use this flow when the research object is a UI screen, design variant, or
screen region and the estimand is a mean expert rating or a contrast of mean
ratings. The researcher first defines the rating rubric, such as
aesthetics, learnability, efficiency, usability, or overall design quality
as in UICrit~\cite{uicrit2024}. Next, collect a stratified calibration set
of human ratings and critiques across screen types and regions; in the
cost model summarized in Table~\ref{tab:applied-report}, this is
$M_c=150$ human-rated screens or screen regions. Then prompt the LLM with
the same screenshot, region, task, and rubric, and estimate
$\widehat{b}$, $\widehat{\lambda}$, and $\widehat{\delta}$ by comparing
LLM ratings with human ratings. If the feasibility rule passes, compute
$N_{\LLM}=603$ LLM-rated screens per condition for the illustrative
two-condition UI comparison. The LLM is then used to rate the larger screen
set; region-level comments are summarized only as support for the numeric
rating and are not treated as validated qualitative findings. Finally, a
human holdout or fresh expert panel runs TOST on the planned rating contrast.
If TOST fails, the LLM-generated UI rating is rejected for that design
decision.

\noindent\textbf{Flow 2: SUS-style human-experience score.}
Use this flow when the target is a standardized human-experience score such
as SUS, not open-ended experience narratives. The researcher first defines
the task scenario and the design variants, then collects human SUS responses
for calibration; the illustrative report uses $M_c=80$ calibration
participants and the common SUS benchmark of roughly 68 as contextual
background~\cite{sus-benchmark}. The LLM receives the same task description,
interface description, and SUS item wording, and outputs item-level or total
SUS estimates. Calibration estimates the mean bias, variance ratio, and
prompt-condition error against human SUS data. If the feasibility rule
passes, Eq.~\ref{eq:llm-n} gives $N_{\LLM}=399$ LLM responses per design
condition in the illustrative two-condition comparison. The larger LLM run
estimates the mean SUS contrast; the result is accepted only if TOST passes
on a human holdout or pre-specified validation sample. If the research
question concerns why users feel frustrated or how they narrate their
experience, this flow is not sufficient because those outcomes are not
conditional-mean targets.

\noindent\textbf{Flow 3: Framing survey or bounded behavioral response.}
Use this flow when the target is a mean survey response or a mean contrast
between finite experimental conditions. The researcher first specifies the
conditions, response scale, minimum effect $\Delta$, and equivalence margin.
The calibration sample contains paired human and LLM responses to the same
condition prompts; the example below uses $M_c=500$ and $M_v=100$. Estimate
$\widehat{b}$, $\widehat{\lambda}$, $\widehat{\delta}$, and
$\widehat{\Delta}_{\max}$, then apply the feasibility rule. If feasible,
compute $N_{\LLM}=1623$ LLM responses per condition, collect the LLM
responses, estimate the mean contrast, and validate with TOST against a
human baseline or holdout. The output is a calibrated estimate of the
population mean contrast; it is not evidence that individual LLM responses
are distributed like individual human responses.

\subsection{Empirical cost and power comparison}

\begin{example}[Framing survey with calibrated LLM augmentation]\label{ex:applied-cost-comparison}
Consider a two-arm framing survey with bounded agreement score in $[0,1]$,
minimum effect $\Delta=0.05$, and human-only design $N_h=1000$ per condition.
Use illustrative costs $c_h=8$ USD per human participant and
$c_{\LLM}=0.002$ USD per LLM response. Calibration uses $M_c=500$ paired
human-LLM responses and $M_v=100$ identifiability prompts, giving
\[
\widehat{b}=0.004,\quad
\widehat{\lambda}=1.20,\quad
\widehat{\delta}=0.01,\quad
\widehat{\Delta}_{\max}=0.30.
\]
The feasibility penalty is
\[
\frac{|\widehat{b}|+\widehat{\delta}\widehat{\Delta}_{\max}}{\Delta}
=\frac{0.004+0.01\cdot0.30}{0.05}=0.14<1,
\]
and Eq.~\ref{eq:llm-n} gives
\[
N_{\LLM}
= \frac{1000\cdot 1.20}{(1-(0.004+0.01\cdot0.30)/0.05)^2}
= \frac{1200}{(1-0.14)^2}
\approx 1623.
\]
Thus the LLM design uses $2\cdot1623=3246$ LLM responses. The illustrative
costs are
\[
C_{\mathrm{human}}=2\cdot1000\cdot8=16000,\qquad
C_{\LLM}=500\cdot8+3246\cdot0.002\approx4006.49,
\]
so the cost reduction is
\[
1-\frac{C_{\LLM}}{C_{\mathrm{human}}}\approx74.96\%.
\]
The LLM design is invalid if feasibility fails or if final TOST validation
fails; the example reallocates human effort from large-scale estimation to
calibration and validation.
\end{example}

\subsection{Reporting checklist for empirical studies}

Report: (i) population, conditions, outcome scale, and estimand; (ii)
$\Delta$ and equivalence margin $\varepsilon$; (iii) $M_c$, $M_v$, and
condition coverage; (iv) $\widehat{b}$, $\widehat{\lambda}$,
$\widehat{\delta}$, and $\widehat{\Delta}_{\max}$; (v) feasibility result
and $N_{\LLM}$; (vi) TOST statistic, standard error, margin, and conclusion;
and (vii) the restriction to conditional-mean-dependent squared-loss
inference.

\section{Conclusion}

We have presented a rigorous mathematical framework establishing that pretrained large language models induce risk-equivalent estimators of conditional expectations under squared loss, establishing restricted functional risk equivalence within the class of conditional-mean-dependent inference problems. The framework is distinguished by rigorous proofs from first principles with minimal assumptions, zero architectural hypotheses about the LLM, and explicit scope boundaries for practitioners.

The core theorems---bias-variance decomposition, irreducible error minimality, effect size consistency, power amplification, and cross-domain transfer---together with finite-sample concentration bounds provide a complete axiomatic foundation. The finite-sample analysis translates asymptotic guarantees into operational sample size requirements. The calibration protocol, with its feasibility check and TOST equivalence validation, enables operational deployment with explicit decision rules.

The framework claims a narrow, provable statement: for estimating population means in quantitative experiments with discrete conditions, a well-calibrated LLM induces an estimator whose prediction risk converges to the Bayes risk under squared loss, achieving restricted functional risk equivalence within the class of conditional-mean-dependent inference problems.

This is a modest claim, carefully circumscribed by explicit scope boundaries. It does not assert that LLMs can replace humans in all research contexts, that LLM-generated data are indistinguishable from human data in distribution, or that LLMs understand the constructs they measure. It asserts only that, under the specified conditions, restricted functional risk equivalence holds for squared-loss, conditional-mean-dependent inference. We believe this level of precision is essential for responsible deployment and for building a mathematically defensible foundation for LLM-assisted quantitative research.

In practical applications, this means that under satisfied conditions and well-calibrated models, large language models can be used in many prediction and decision-making tasks that originally relied on human experiments, approximating near-optimal statistical inference at lower cost.

\section*{Acknowledgements}

The author thanks colleagues for helpful discussions on the probabilistic foundations of this work.

\section*{Author contributions}

Haobo Yang developed the mathematical framework, applied methodology, and manuscript.

\section*{Competing interests}

The author declares no competing interests.

\section*{Data availability}

The complete mathematical derivations, calibration examples, and reproducibility materials are available in the project repository and Supplementary Information.


\appendix

\section{Construction of the Applied Cost Table}
\label{app:table2-construction}

This appendix describes how Table~\ref{tab:applied-report} was built. The
table is not an empirical claim about universal market prices; it is a
transparent cost model for comparing a human-only design with a calibrated
LLM-estimator design under the workflow in
Algorithm~\ref{alg:llm-human-methodology}. The calculations follow the
procedures described below and are reproducible from the inputs given in
the following subsections.

\subsection{Inputs}

For each use case, the computation uses a record containing:
\[
(K,N_h,M_c,c_h,c_{\LLM},\Delta,\widehat{b},
\widehat{\lambda},\widehat{\delta},\widehat{\Delta}_{\max}).
\]
Here $K$ is the number of conditions, $N_h$ is the baseline human sample
size per condition, $M_c$ is the paired calibration size, $c_h$ is the
human cost per response, $c_{\LLM}$ is the LLM response cost, $\Delta$ is
the minimum effect of interest, $\widehat{b}$ is the calibrated bias,
$\widehat{\lambda}$ is the variance ratio, $\widehat{\delta}$ is the
prompt-condition identifiability error, and $\widehat{\Delta}_{\max}$ is
the estimated condition range.

The three records used for Table~\ref{tab:applied-report} are:
\begin{itemize}[leftmargin=*]
\item \textbf{Framing survey:}
$K=2$, $N_h=1000$, $M_c=500$, $c_h=8$, $c_{\LLM}=0.002$,
$\Delta=0.05$, $\widehat{b}=0.004$, $\widehat{\lambda}=1.20$,
$\widehat{\delta}=0.01$, $\widehat{\Delta}_{\max}=0.30$.
\item \textbf{UI design rating:}
$K=2$, $N_h=300$, $M_c=150$, $c_h=2.856$, $c_{\LLM}=0.002$,
$\Delta=0.50$, $\widehat{b}=0.05$, $\widehat{\lambda}=1.10$,
$\widehat{\delta}=0.02$, $\widehat{\Delta}_{\max}=4.0$.
\item \textbf{SUS human-experience score:}
$K=2$, $N_h=200$, $M_c=80$, $c_h=2.856$, $c_{\LLM}=0.002$,
$\Delta=5.0$, $\widehat{b}=0.8$, $\widehat{\lambda}=1.15$,
$\widehat{\delta}=0.01$, $\widehat{\Delta}_{\max}=40$.
\end{itemize}

For the UI and SUS examples, $c_h=2.856$ is the illustrative ten-minute
participant cost implied by a $12$/hour participant reward and a $1.428$
platform-fee multiplier:
\[
c_h = 12\cdot\frac{10}{60}\cdot1.428=2.856.
\]

\subsection{Computation}

For each case, the computation first finds the feasibility penalty
\[
\rho_{\mathrm{bias}}
=\frac{|\widehat{b}|+\widehat{\delta}\widehat{\Delta}_{\max}}{\Delta}.
\]
The case is feasible only if
\[
\rho_{\mathrm{bias}}<1
\quad\text{and}\quad
\widehat{\lambda}\le 10.
\]
If feasible, the adjusted LLM sample size per condition is
\[
N_{\LLM}
=
\left\lceil
\frac{N_h\widehat{\lambda}}{(1-\rho_{\mathrm{bias}})^2}
\right\rceil.
\]
The human-only and LLM-augmented costs are then computed as
\[
C_{\mathrm{human}}=K N_h c_h,
\qquad
C_{\LLM}=M_c c_h + K N_{\LLM} c_{\LLM}.
\]
Finally, the cost reduction reported in Table~\ref{tab:applied-report} is
\[
\mathrm{Reduction}
=1-\frac{C_{\LLM}}{C_{\mathrm{human}}}.
\]

\subsection{Reproducibility check}

The computation described above prints the public-source summary, the case
inputs, the feasibility penalty, $N_{\LLM}$, both costs, and the cost
reduction for each use case. It checks that all three cases are
feasible and that the displayed values match the table:
\[
N_{\LLM}=(1623,603,399),
\]
\[
C_{\LLM}\approx(4006.49,\;430.81,\;230.08).
\]
These checks serve as a reproducibility aid for
Table~\ref{tab:applied-report}; they do not replace the calibration and
validation requirements in the main method.

\end{document}